\documentclass{article}
\usepackage{authblk}
\PassOptionsToPackage{numbers, compress}{natbib}

\newif\ifISWORD
\ISWORDtrue
\ISWORDfalse

\newif\ifISARXIV
\ISARXIVtrue
\ISARXIVfalse

\newif\ifSHOWNEWWORK
\SHOWNEWWORKtrue
\SHOWNEWWORKfalse

\usepackage[final]{neurips_2021}

\usepackage[utf8]{inputenc} 
\usepackage[T1]{fontenc}    
\usepackage{hyperref}       
\usepackage{url}            
\usepackage{booktabs}       
\usepackage{amsfonts}       
\usepackage{nicefrac}       
\usepackage{microtype}      
\usepackage{xcolor}         

\title{Greedy-Step Off-Policy Reinforcement Learning}

\label{User-Defined-Begin}

\usepackage{usebib}
\newbibfield{author} 
\bibinput{GreedyMultistep}

\usepackage{amsmath,amsfonts,bm}

\def\eqref#1{equation~\ref{#1}}

\def\1{\bm{1}}

\DeclareMathAlphabet{\mathsfit}{\encodingdefault}{\sfdefault}{m}{sl}
\SetMathAlphabet{\mathsfit}{bold}{\encodingdefault}{\sfdefault}{bx}{n}

\newcommand{\E}{\mathbb{E}}

\newcommand{\R}{\mathbb{R}}

\DeclareMathOperator*{\argmax}{arg\,max}

\usepackage{xcolor}

\definecolor{usccardinal}{rgb}{0.6, 0.0, 0.0}
\newif\ifshowtmp
\showtmptrue
\showtmpfalse
\ifshowtmp
	
	\newcommand{\tore}[1]{\textcolor{red}{#1}}
	\newcommand{\ttt}[1]{\textcolor{usccardinal}{#1}}
	\newcommand{\T}[1]{\textcolor{usccardinal}{#1}}
\else
	\newcommand{\ttt}{}
	\newcommand{\T}{}

	\newcommand{\tore}{}
\fi

\newif\iffastcompile
\fastcompiletrue
\fastcompilefalse

\newif\ifshowfig
\showfigtrue

\newif\ifshowstruct
\showstructtrue
\showstructfalse 

\ifshowstruct
	\newcommand{\tstructz}[1]{[#1] }
\else
	\newcommand{\tstructz}[1]{}
\fi

\label{packages}
\usepackage{mfirstuc}
\usepackage{graphicx}
\usepackage{subfig}
\usepackage{caption}
\usepackage{multirow}
\usepackage{multicol}
\usepackage{array}
\usepackage{booktabs}

\usepackage{amsmath}
\usepackage{amssymb, graphicx, url, mathtools}
\usepackage[overload]{empheq}

\usepackage{wrapfig}

\usepackage{algorithmic,algorithm}    
\usepackage{xr}
\usepackage{footnote}
\makesavenoteenv{tabular}
\makesavenoteenv{table}
\usepackage{stfloats}
\usepackage{pdfpages}
\usepackage{bbm}
\usepackage{textcomp}
\usepackage{chngcntr}
\usepackage{verbatim}

\usepackage{xparse,calc}
\usepackage{afterpage}

\label{Command}

\newtheorem{theorem}{Theorem}
\newtheorem{lemma}{Lemma}

\counterwithin*{question}{section}

\newtheorem{innercustomgeneric}{\customgenericname}
\providecommand{\customgenericname}{}

\newcommand{\newcustomtheoremRef}[2]{%
  \newenvironment{#1}[1]
  {%
   \renewcommand\customgenericname{#2}%
   \renewcommand\theinnercustomgeneric{\ref{##1}}%
   \innercustomgeneric
  }
  {\endinnercustomgeneric}
}
\newcustomtheoremRef{lemmaCustomRef}{Lemma}
\newcustomtheoremRef{theoremCustomRef}{Theorem}
\newcustomtheoremRef{propertyCustomRef}{Property}

\newenvironment{proof}[1][]{{\noindent \textit{Proof#1:}~}}{\hfill $\square$}

\newcommand{\EE}{\mathbb{E}}

\newcounter{numAuthors}
\usepackage{xstring}

\renewcommand{\citeauthor}[1]{%
     \edef\expAuthors{\usebibentry{#1}{author}}%
        \setcounter{numAuthors}{1}
        \def\numAuthorsInt{}%
        \StrCount{\expAuthors}{and}[\numAuthorsInt]%
        \addtocounter{numAuthors}{\numAuthorsInt}%
        \ifthenelse{\value{numAuthors} = 1}{%
            \StrBefore{\expAuthors}{,}~\cite{#1}%
        }%
        {\ifthenelse{\value{numAuthors} = 2}{%
            \def\secondAuthInt{}%
            \StrBehind{\expAuthors}{and}[\secondAuthInt]%
            \StrBefore{\expAuthors}{,}~and \StrBefore{\secondAuthInt}{,}~\cite{#1}%
        }{%
            \StrBefore{\expAuthors}{,}~et~al.\@{}~\cite{#1}%
        }}%
}

\newcolumntype{+}{>{\global\let\currentrowstyle\relax}}
\newcolumntype{Y}{>{\currentrowstyle}}\label{key}
\newcolumntype{-}{>{\currentrowstyle}}

\newcommand{\breakingcomma}{
  \begingroup\lccode`~=`,
  \lowercase{\endgroup\expandafter\def\expandafter~\expandafter{~\penalty0 }}
}

\DeclareFontFamily{U}{mathx}{\hyphenchar\font45}
\DeclareFontShape{U}{mathx}{m}{n}{<-> mathx10}{}
\DeclareSymbolFont{mathx}{U}{mathx}{m}{n}
\DeclareMathAccent{\widebar}{0}{mathx}{"73}

\label{Settings}
\usepackage{hyperref}
\usepackage{cleveref}
\hypersetup{
	colorlinks,
	linkcolor={black},
	citecolor={black},
	urlcolor={black}
}

\newcommand{\checkmarkcolored}{ $\color{blue}{\checkmark}$ }
\newcommand{\notcheckmarkcolored}{ $\color{red}{\times}$ }

\usepackage{makecell}

\Crefname{figure}{Fig.}{Fig.}
\Crefname{equation}{Eq.}{Eq.}

\def\policySetText/{set of behavior policies}
\def\policyDistText/{behavior policies selection distribution}
\NewDocumentCommand{\policyDistMath}{ O{} O{} }{ \mathcal{P}^{#1}_{#2} }
\NewDocumentCommand{\policySet}{ O{} O{} }{ \widehat{\Pi} }

\def\policySetTextMath/{\policySetText/ $\policySet$}
\def\policyDistTextMath/{\policyDistText/ $\policyDistMath$}

\NewDocumentCommand{\characteristic}{ O{} }{\emph{(#1)}}

\NewDocumentCommand{\traj}{ O{\pi} O{s_t,a_t} }{ {\tau}_{#2}^{#1} }

\NewDocumentCommand{\nstepReturn}{ O{N} O{\traj} O{Q} }{ {R}_{#1}^{#3}({#2}) }

\def\step/{bootstrapping step}
\def\steps/{bootstrapping steps}
\def\lanpaper/{[Lan, 2020]}
\def\youngpaper/{[Young, 2019]}

\newcommand{\rls}{  \mathcal{S}} 
\newcommand{\rla}{  \mathcal{A}}
\definecolor{colorQ}{RGB}{0,176,80}

\def\greedyStepQLearning/{Greedy-Step Q learning}
\def\greedyStepDQN/{Greedy-Step DQN}

\def\greedyStepQLearningFull/{Greedy-Step Q learning}
\def\greedyStepDQNFull/{Greedy-Step DQN}

\def\greedyDoubleExpectedStepQLearningFull/{Greedy-Expected-Step Q learning}
\def\greedyDoubleExpectedStepQLearning/{Greedy-Expected-Step Q learning}

\def\proposedMethod/{\proposedMethodPart/Q learning}

\def\proposedMethodDeepFull/{\proposedMethodPartFull/ DQN}
\def\proposedMethodDeep/{\proposedMethodPart/DQN}

\def\multiStepOperatorText/{multi-step optimality operator}
\NewDocumentCommand{\multiStepOperatorMath}{ O{N} O{\policySet} }{ \mathcal{B}^{{#1}}_{#2} }

\NewDocumentCommand{\multiStepOnPolicyOperatorMath}{ O{N} O{\pi} }{ \widebar{\mathcal{B}}^{{#1}}_{#2} }
\NewDocumentCommand{\multiStepOffPolicyOperatorMath}{ O{N} O{\pi,\policyDistMath} }{ \widebar{\mathcal{B}}^{{#1}}_{#2} }

\def\multiStepOperatorTextMath/{\multiStepOperatorText/ $\multiStepOperatorMath$}

\def\OneStepOperatorText/{Bellman Optimality Operator}
\def\oneStepOperatorText/{Bellman Optimality Operator}
\NewDocumentCommand{\oneStepOperatorMath}{ O{}  }{ \multiStepOperatorMath[#1][] }

\def\oneStepOperatorTextMath/{\oneStepOperatorText/ $\oneStepOperatorMath$}

\def\greedyReturn/{greedy-step return}
\def\GreedyReturn/{Greedy-Step Return}
\def\Greedyreturn/{Greedy-step return}

\def\GreedyStepPrefix/{Greedy-Step}
\def\greedyStepPrefix/{Greedy-Step}
\def\greedyStepPrefixFull/{Greedy-Step}
\def\greedyStepValueIteration/{\greedyStepPrefixFull/ Value Iteration}
\def\greedyStepEquationFull/{\greedyStepPrefixFull/ Bellman Optimality Equation}
\def\greedyStepEquation/{\greedyStepPrefix/ Equation}
\def\greedyStepOperatorText/{\greedyStepPrefix/ Operator}
\def\greedyStepOperatorTextFull/{\greedyStepPrefixFull/ Bellman Optimality Operator}
\NewDocumentCommand{\greedyStepOperatorMath}{ O{N} O{\policySet} }{ \mathcal{G}^{ {#1} }_{ {#2} } }
\NewDocumentCommand{\greedyStepOperatorMathDeterministicContinuousMDP}{ O{N} O{\policySet} }{ \mathcal{G}^{ {#1} }_{ {#2} } }
\def\greedyStepOperatorTextMath/{\greedyStepOperatorText/ $\greedyStepOperatorMath$}

\NewDocumentCommand{\greedyStepRecombinationOperatorMath}{ O{N} }{ \mathcal{G}^{{#1}} }

\def\greedyExpectedStepPrefix/{Greedy-Step}
\def\greedyExpectedStepPrefixFull/{Greedy-Step}
\def\greedyExpectedStepValueIteration/{\greedyExpectedStepPrefixFull/ Value Iteration}
\def\greedyExpectedStepEquationFull/{\greedyExpectedStepPrefixFull/ Bellman Optimality Equation}
\def\greedyExpectedStepEquation/{\greedyExpectedStepPrefix/ Equation}
\def\greedyExpectedStepOperatorText/{\greedyExpectedStepPrefix/ Operator}
\def\greedyExpectedStepOperatorTextFull/{\greedyExpectedStepPrefixFull/ Bellman Optimality Operator}
\NewDocumentCommand{\greedyExpectedStepOperatorMath}{ O{N} O{\policySet} }{ {\mathcal{G}}^{ {#1} }_{ {#2} }  }
\def\greedyExpectedStepOperatorTextMath/{\greedyExpectedStepOperatorText/ $\greedyExpectedStepOperatorMath$}

\def\greedyDoubleExpectedStepPrefix/{Greedy-Step}
\def\greedyDoubleExpectedStepPrefixFull/{Greedy-Step}
\def\greedyDoubleExpectedStepValueIteration/{\greedyDoubleExpectedStepPrefixFull/ Value Iteration}
\def\greedyDoubleExpectedStepEquation/{\greedyDoubleExpectedStepPrefixFull/ Equation}
\def\greedyDoubleExpectedStepEquationFull/{\greedyDoubleExpectedStepPrefixFull/ Bellman Optimality Equation}
\def\greedyDoubleExpectedStepOperatorText/{\greedyDoubleExpectedStepPrefix/ Operator}
\def\greedyDoubleExpectedStepOperatorTextFull/{\greedyDoubleExpectedStepPrefixFull/ Bellman Optimality Operator}
\NewDocumentCommand{\greedyDoubleExpectedStepOperatorMath}{ O{N} O{\policySet} }{  {{\mathcal{G}}}^{ {#1} }_{ {#2} } }
\def\greedyDoubleExpectedStepOperatorTextMath/{\greedyDoubleExpectedStepOperatorText/ $\greedyDoubleExpectedStepOperatorMath$}

\NewDocumentCommand{\K}{ O{k} }{ _{#1} }

\definecolor{colortrajzero}{RGB}{0,100,0}
\NewDocumentCommand\trajzero{ O{ } }{{\color{colortrajzero}{trajectory ${ {\tau_0} }$#1}}} %

\definecolor{colortrajone}{RGB}{0,0,139}
\NewDocumentCommand\trajone{ O{ } }{{\color{colortrajone}{trajectory $\tau_1$#1}}} %

\def\colortrajfinal{brown}
\definecolor{colorQupdate}{RGB}{147,112,219}

\def\iteration/{iteration}
\def\Iteration/{Iteration}

\NewDocumentCommand\trajlegend{ O{0}  O{0.03} O{0.03} O{} O{\scriptsize} O{\rm final} }{ 
	\begin{minipage}{1.\linewidth}
		#5
		\ifthenelse{ \equal{#1}{0} }{}{\hspace{#1\linewidth}}
		\mbox{\includegraphics[width=#3\linewidth]{figs/fig_traj_arrows.pdf}
		\trajzero[ using $\pi_0$]
		}
		\hspace{#2\linewidth}
		\mbox{
		\includegraphics[width=#3\linewidth,page=2]{figs/fig_traj_arrows.pdf}
		\trajone[ using $\pi_1$]
		}
		\hspace{#2\linewidth}
		\mbox{
		\includegraphics[width=#3\linewidth,page=3]{figs/fig_traj_arrows.pdf}
		{\color{\colortrajfinal}{trajectory using $\pi^{Q^{#6}}$}}
		}
		\ifthenelse{ \equal{#4}{} }{\hspace{#2\textwidth} $\color{colorQupdate}{\blacksquare}$ changed $Q$ \hspace{#2\textwidth}}{}		 
	\end{minipage}
}

\def\nrandomseed/{5}

\usepackage{pgf,tikz}
\usetikzlibrary{arrows}
\label{User-Defined-End}

\author[ ]{Yuhui Wang}
\author[ ]{Qingyuan Wu}
\author[ ]{Pengcheng He}
\author[ ]{Xiaoyang Tan}

\affil[ ]{{\tore{College of Computer Science and Technology,Nanjing University of Aeronautics and Astronautics}}}
\affil[ ]{MIIT Key Laboratory of Pattern Analysis and Machine Intelligence}
\affil[ ]{\textit {\{y.wang, wuqingyuan, yukina233, x.tan\}@nuaa.edu.cn}}

\begin{document}

\maketitle


\begin{abstract}


Most of the policy evaluation algorithms are based on the theories of Bellman Expectation and Optimality Equation, which derive two popular approaches - Policy Iteration (PI) and Value Iteration (VI). However, multi-step bootstrapping is often at cross-purposes with and off-policy learning in PI-based methods due to the large variance of multi-step off-policy correction. In contrast, VI-based methods are naturally off-policy but subject to one-step learning.
In this paper, we deduce a novel multi-step Bellman Optimality Equation by utilizing a latent structure of multi-step bootstrapping with the optimal value function. Via this new equation, we derive a new multi-step value iteration method that converges to the optimal value function with exponential contraction rate $\mathcal{O}(\gamma^n)$ but only linear computational complexity. 
Moreover, it can naturally derive a suite of multi-step off-policy algorithms that can safely utilize data collected by arbitrary policies without correction. 
Experiments reveal that the proposed methods are reliable, easy to implement and achieve state-of-the-art performance on a series of standard benchmark datasets.

\end{abstract}


\section{Introduction}

\T{Multi-step reinforcement learning (RL) is a set of methods that aim to adjust the trade-off of utilization between the observed data of rewards and the knowledge of future return.
Recent advances on multi-step RL have achieved remarkable empirical success \citep{horgan2018distributed, barth2018distributed}.} 
However, one major challenge of multi-step RL comes from how to achieve the right balance between the two terms.
Such a challenge can be regarded as a kind of data-knowledge trade-off in some sense.
Particularly, a large bootstrapping step tends to quickly propagate the information in the data, while a small one relies more on the knowledge stored on the learned value function. 
The classical solution to address this issue is to impose a fixed prior distribution over every possible step, e.g., TD($\lambda$) \citep{sutton2018reinforcement}, GAE($\lambda$) \citep{Schulman2016HighDimensional}. Such a solution often ignores the quality of data and on-going knowledge, which dynamically improves over the learning process. Besides, the prior distribution usually has to be tuned case by case.

Another issue related to multi-step RL is off-policy learning, i.e., its capability to learn from data from other behavior policies.  
Previous research on this is mainly conducted under the umbrella of Policy Iteration (PI) \citep{sutton2018reinforcement}, with the goal to evaluate the value function of a target policy \citep{precup2000eligibility,harutyunyan2016q,munos2016safe, sutton2018reinforcement, Schulman2016HighDimensional}.
Despite their success, those methods usually suffer from certain undesired side effects of off-policy learning, e.g., high variance due to the product of importance sampling (IS) ratios, and the restrictive premise of being able to access both the behavior and the target policy (to compute the IS ratios). 
In contrast with PI, Value Iteration (VI) methods aim to approximate the optimal value function, by propagating the value of the most promising action \citep{sutton2018reinforcement,szepesvari2010algorithms}.
The good side is that these methods can safely use data from any behavior policy without any correction. However, it needs to conduct value propagation step-by-step, making them somewhat unnatural for multi-step learning.


In this paper, we aim to accelerate value iteration via multi-step bootstrapping.
Specially, we deduce a novel multi-step Bellman Optimality Equation by utilizing a latent structure of multi-step bootstrapping with the optimal value function. 
Via this new equation, a novel multi-step value iteration method, named \emph{\greedyStepValueIteration/}, is derived.
We theoretically show that the new value iteration method converges to the optimal value function.
Moreover, the contraction rate is $\mathcal{O}(\gamma^{n})$ under some condition ( and $\mathcal{O}(\gamma^{2})$ with a loose condition; $\mathcal{O}(\gamma^{1})$ in the worst case), while the computational complexity is linear.
To the best of our knowledge, this is the first multi-step value iteration method that converges to the optimal value function with an exponential rate while with linear computation complexity.
{Besides}, we theoretically show that the new operator is unbiased with data collected by any behavior policy.
Therefore, it naturally derives multi-step off-policy RL algorithms which are able to utilize off-policy data collected by the arbitrary policy safely.
Moreover, experiments reveal that the proposed algorithms are reliable, easy to implement, and achieve state-of-the-art performance on a series of standard benchmark datasets.

\section{Preliminaries}


A \emph{Markov Decision Processes (MDP)} is described by the tuple $(\mathcal{S},\mathcal{A},{\cal T},r,\gamma)$. 
$\mathcal{S}$ is the state space; $\mathcal{A}$ the action space; $\gamma \in (0,1)$ is the discount factor; 
${\cal T}$ is the transition function mapping $s,a \in {\rls} \times {\rla}$ to distributions over $\rls$, 
denoted as ${\cal T}(\cdot|s, a)$.
$r $ is the reward function mapping $s,a$ to distribution over $\R$, 
denoted as $r(\cdot|s, a)$.
A \emph{deterministic/stochastic MDP} is defined as a MDP with deterministic/stochastic transition function $\mathcal{T}$ and reward function ${r}$.
A \emph{discrete/continuous MDP} is defined as a MDP with discrete/continuous state space $\mathcal{S}$ and action space $\mathcal{A}$.


Our goal is to find actions which can bring largest return, which is defined as the accumulated reward from timestep $t$, i.e., $R_t^\gamma = \sum_{n=0}^{\infty} \gamma^n r_{t+n} $.
The value function of a policy $\pi$ is defined as the expected return by executing policy $\pi$, $Q^\pi (s,a) \triangleq \E \left[ R_t^\gamma  | s_t=s,a_t=a, \pi \right] $. 
The Bellman Expectation Operator is defined as
\begin{equation}
\begin{aligned}
 (\oneStepOperatorMath[\pi] Q )(s,a) \triangleq  
	\EE_{ r_{t}, s_{t+1}, a_{t+1} \sim \pi } 
	\left[
	\left.
	 r_t +
	 \gamma Q( s_{t+1}, a_{t+1} )
	\right|
	\parbox{0.4in}{
	\noindent
	$s_t=s$\\$a_t=a$
	}
	\right]
\end{aligned}
\end{equation}
Value-based RL methods aims to approximate the \emph{optimal value function} $Q^*(s,a) = \max_{\pi} Q^\pi(s,a) $. 
Solving the optimal value function $Q^*$ is a non-trivial task. This task can be achieved by iteratively applying 
\emph{Bellman Optimality Operator}, which is defined as
\begin{equation}\label{eq_oneStepOperator}
\begin{aligned}
 (\oneStepOperatorMath Q )(s,a) \triangleq  
	\EE_{ r_{t}, s_{t+1} } 
	\left[
	\left.
	 r_t +
	 \gamma \max_{a^\prime_{t+1}} Q( s_{t+1}, a^\prime_{t+1} )
	\right|
	\parbox{0.4in}{
	\noindent
	$s_t=s$\\$a_t=a$
	}
	\right]
\end{aligned}
\end{equation}
The optimal value function $Q^*$ satisfies the Bellman Optimality Equation, $Q^*=\oneStepOperatorMath Q^*$, with the equality holds component-wise.

\section{\greedyStepValueIteration/} \label{sec_equation_main}

Value iteration looks forward for one-step, and then chooses the largest estimated value among various actions to update the value function. 
\begin{equation}\label{eq_oneStepOperator}
\begin{aligned}
Q_{k+1}(s,a) = 
\EE_{  s^\prime } 
	\left[
	r(s,a) + 
	 \gamma \max_{a^\prime_{}} Q_k( s^\prime, a^\prime )
	\right]
\end{aligned}
\end{equation}

We aim to accelerate the process of value iteration by rollout behavior policy for multiple steps.
Our new method, named \emph{\greedyStepValueIteration/}, is updated in the following way.
\begin{equation}\label{eq_greedyStepOperator}
\begin{aligned}
Q_{k+1}\left( s_0,a_0 \right) = \mathbb{E} _{s_1}\left[ \max_{\pi \in \widehat{\Pi }} \max_{1\le n\le N} \mathbb{E} _{\tau ^{1:n}_{s_1}\sim \pi}\left[ \sum_{t=0}^{n-1}{\gamma ^t}r\left( s_t,a_t \right) +\gamma ^n\max_{a_{n}'} Q_k\left( s_n,a_{n}' \right) \right] \right] 
\end{aligned}
\end{equation}
where 
$\policySet$ is the \emph{\policySetText/}; $\tau ^{1:n}_{s_1}=(a_1,s_2,a_2,\cdots,s_n)$ is the trajectory starting from $s_1$; and $\tau ^{1:n}_{s_1} \sim \pi$ means sample trajectory according policy $\pi$.
As the equation implies, it looks forward by rolling-out various behavior policies with multiple steps,
then greedily chooses the largest estimated value among various behavior policies, actions, and steps.
The intuition behind this operation is that, it can quickly look forward to acquire future high-credit information.
In this way, it can bridge a highway allowing unimpeded flow of high-credit information across horizons.
\Cref{alg_GreedyMultistepValueIteration} shows the \greedyStepValueIteration/ algorithm.

\ifISWORD
\else
\begin{algorithm}[t]
\caption{ \greedyStepValueIteration/ }
\begin{algorithmic}\label{alg_GreedyMultistepValueIteration}
\STATE \textbf{Input:} 
\policySetTextMath/; the maximal \step/ $N$.
\STATE \textbf{Initialize:} initial value function $Q\K[0] \in \R^{|\rls|\times |\rla|} $, $\epsilon$, $k=0$.
\REPEAT
	\STATE $Q\K[k+1] \leftarrow \greedyStepOperatorMath Q\K[k]$ 
	\STATE $k \leftarrow k+1$
\UNTIL{ $ \| Q\K[k]-Q\K[k-1] \|_\infty \leq \epsilon $ }
\STATE \textbf{Output:} $Q\K[k]$
\end{algorithmic}
\end{algorithm}
\fi

This new value iteration converges to the fixed point of the optimal value function $Q^*$, which leads to a novel multi-step Bellman Optimality Equation. 
Let us define the \greedyStepOperatorTextFull/ $\greedyStepOperatorMath$ (\greedyStepOperatorText/),
$$
\mathcal{G} _{\widehat{\Pi }}^{N}Q\left( s_0,a_0 \right) \triangleq \mathbb{E} _{s_1}\left[ \max_{\pi \in \widehat{\Pi }} \max_{1\le n\le N} \mathbb{E} _{\tau ^{1:n}_{s_1}\sim \pi}\left[ \sum_{t=0}^{n-1}{\gamma ^t}r\left( s_t,a_t \right) +\gamma ^n\max_{a_{n}'} Q\left( s_n,a_{n}' \right) \right] \right] 
$$


\ifISWORD
\else
\begin{figure}[b]
\centering{
{
	\includegraphics[width=.5\linewidth]{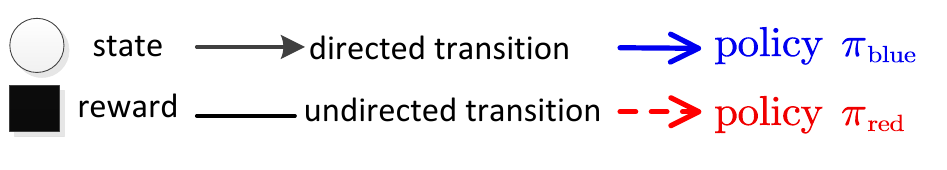} 	}
	\includegraphics[width=1.\linewidth]{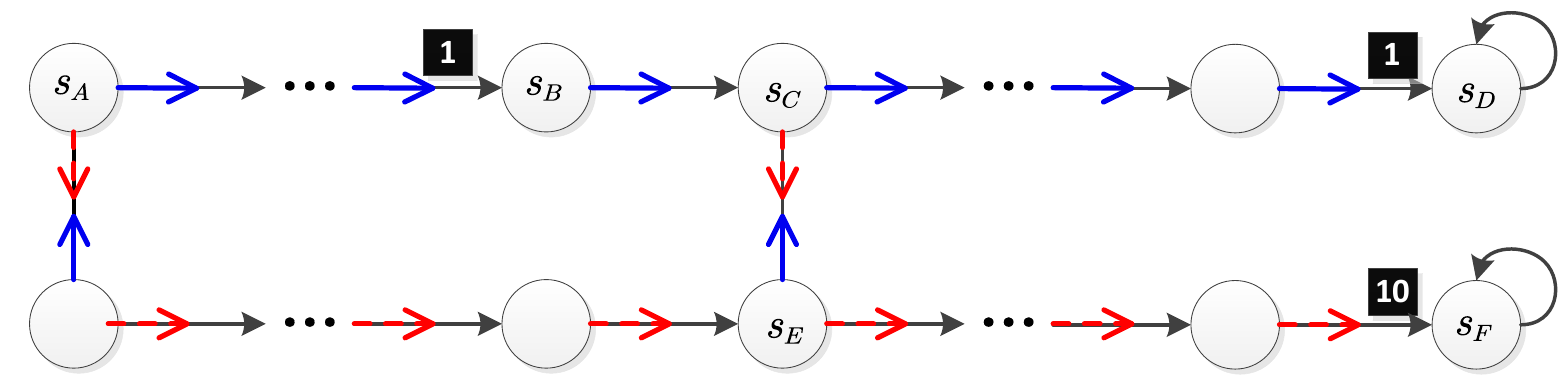} 
	}
	\captionof{figure}{ 
	A simple $N$-horzion MDP problem with two behavior polciies. 
	The horizon between $s_A$ and $s_F$ is $N$.
	The rewards are zero if not specified.
	}\label{fig_example}
\end{figure}
\fi

\begin{theorem}\label{theorem_greedyStepEquation}
(\greedyStepEquationFull/ (\greedyStepEquation/) )
For any \policySetTextMath/, $N\geq 1$, we have $Q^* = \greedyStepOperatorMath Q^*$, i.e.,
\begin{equation}\label{eq_greedyStepEquation}
\begin{small}
Q^*\left( s_0,a_0 \right) = \mathbb{E} _{s_1}\left[ \max_{\pi \in \widehat{\Pi }} \max_{1\le n\le N} \mathbb{E} _{\tau ^{1:n}_{s_1}\sim \pi}\left[ \sum_{t=0}^{n-1}{\gamma ^t}r\left( s_t,a_t \right) +\gamma ^n\max_{a_{n}'} Q^*\left( s_n,a_{n}' \right) \right] \right] 
\end{small}
\end{equation}
\end{theorem}

\begin{proof}
Our operator can be rewritten as
$$
\mathcal{G} _{N}^{\widehat{\Pi }}Q=\max_{\pi \in \widehat{\Pi }} \max_{1\le n\le N} \left( \mathcal{B} ^{\pi} \right) ^{n-1}\mathcal{B} Q.
$$

As
$\left( \mathcal{B} ^{\pi} \right) ^{n-1}\mathcal{B} Q^*=\left( \mathcal{B} ^{\pi} \right) ^{n-1}Q^*\le Q^*$ for any $\pi$ (the equivalence holds when $n=1$), we get
$$
\max_{\pi \in \widehat{\Pi }} \max_{1\le n\le N} \left( \mathcal{B} ^{\pi} \right) ^{n-1}\mathcal{B} Q^*=Q^*.
$$

\end{proof}

This equation provides a novel equation for obtaining the optimal value function $Q^*$.

Our method converges faster when the behavior policy performs well over a period. 
We first give an intuitive example here and give a formal analysis in \Cref{sec_theorem}.
Take the $N$-horizon problem in \Cref{fig_example} as an example,
our method only requires $2$ iterations to obtain the optimal value function, while classical value iteration requires $N$ iterations.




\subsection{Theoretical Analysis}\label{sec_theorem}

We make analysis with the following results.
1) Solving the equation leads to the optimal value function;
2) The iteration process with the operators converges to the optimal value function;
3) The operators converge faster than the traditional operator.
We consider finite state and action space.
\emph{Note that, all these results hold for any \policySetTextMath/, $N\geq 1$.}


\textbf{{First}}, we show that solving the \greedyStepEquation/ leads to the optimal value function.
\begin{lemma}\label{theorem_contraction}
(Contraction)
For any value function $q, q' \in  \R ^{ |\rls||\rla| } $, we have
$
	\|  \greedyStepOperatorMath q - \greedyStepOperatorMath q' \|
	\leq \| \oneStepOperatorMath q - \oneStepOperatorMath q' \|
	 \leq \gamma 	\|   q -  q' \|
$.
\end{lemma}

\begin{proof}
\begin{align*}&\quad
\left\| \mathcal{G} _{\widehat{\Pi }}^{N}q-\mathcal{G} _{\widehat{\Pi }}^{N}q' \right\| 
\\ &
=\left\| \max_{\pi \in \widehat{\Pi }} \max_{1\le n\le N} \left( \mathcal{B} ^{\pi} \right) ^{n-1}\mathcal{B} q-\max_{\pi \in \widehat{\Pi }} \max_{1\le n\le N} \left( \mathcal{B} ^{\pi} \right) ^{n-1}\mathcal{B} q' \right\| 
\\ &
\le \max_{\pi \in \widehat{\Pi }} \max_{1\le n\le N} \left\| \left( \mathcal{B} ^{\pi} \right) ^{n-1}\mathcal{B} q-\left( \mathcal{B} ^{\pi} \right) ^{n-1}\mathcal{B} q' \right\| 
\\ &
\le \max_{\pi \in \widehat{\Pi }} \max_{1\le n\le N} \gamma ^{n-1}\left\| \mathcal{B} q-\mathcal{B} q' \right\| 
\\ &
\le \left\| \mathcal{B} q-\mathcal{B} q' \right\| 
\\ &
\le \gamma \left\| q-q' \right\| 
\end{align*}
\end{proof}

\begin{theorem}\label{theorem_solution}
(\greedyStepEquationFull/)
For any value function $Q \in \R ^{ |\rls||\rla| } $, we have that $Q=Q^*$ if and only if $Q$ satisfies \greedyStepEquation/, i.e., $Q=\greedyStepOperatorMath Q$.
\end{theorem}


\textbf{{Then}}, we show that the value iteration process with the \greedyStepOperatorTextMath/
converges to the optimal value function $Q^*$.
\begin{theorem}\label{theorem_converge_to_optimal}
$\lim_{k \rightarrow \infty } \|(\greedyStepOperatorMath)^k Q_0 - Q^* \| = 0$.
\end{theorem}

\textbf{{Next}}, we compare the convergence speed of our operator with that of the classical \oneStepOperatorTextMath/.
\begin{theorem}\label{theorem_converge_faster_onestep}
(Faster contraction)
$ \| \greedyStepOperatorMath q - Q^* \|  \leq \| \oneStepOperatorMath q - Q^* \| $.
\end{theorem}
This theorem implies that our operator always converges faster than the \oneStepOperatorText/.

\begin{theorem}\label{theorem_converge_faster_onestep}
(Exponential contraction rate)
Let $\mathcal{S}_{s,n}^{\pi} $ denote the set of states that the policy $\pi$ can visit within $n$ steps, starting from state $s$.
Assume that $q\leq Q^*$.
If for any $s \in \mathcal{S}$ there exists one policy $\pi \in \widehat \Pi $ such that $\pi(s') = \pi^*(s')$ for any $s'\in \mathcal{S}_{s,n}^\pi$, 
then the contraction rate of $\greedyStepOperatorMath$ is $\mathcal{O}(\gamma^{n+1})$,  i.e., $\| \greedyStepOperatorMath q - Q^* \| \leq \gamma^{n+1}\| q - Q^* \| $.\\
Specially, if for any $s \in \mathcal{S}$ there exists one policy  $\pi \in \widehat \Pi $ such that $\pi(s) = \pi^*(s)$, 
then the contraction rate of $\greedyStepOperatorMath$ is $\mathcal{O}(\gamma^2)$.\\
\end{theorem}
This theorem provide a sufficient condition for exponential contraction rate -- if there exists one behavior policy executes the optimal action over a period, then the convergence speed can be fasten.
Note that we require one behavior policy to execute the optimal action one \emph{not} on all states but only on a small subset of states that the optimal policy can visit within a period.

\begin{proof}
First, if
$q\le Q^*$,
then
$\mathcal{G} ^{N}_{\widehat{\Pi }} q\le Q^*,
\mathcal{B} ^nq\le Q^*.
$

Second, if for any $s_1 \in \mathcal{S}$ there exists one policy  $\pi \in \widehat \Pi $ such that $\pi(s_{t}) = \pi^*(s_{t})$ for any $s_{t}$ ($1\leq t\leq n$) in $\tau^{s_1} \sim \pi^*$, then
\begin{align*} & \quad
\mathcal{G} ^{N}_{\widehat{\Pi }} q
\\ &
=\max_{\pi \in \widehat{\Pi }} \max_{1\le n'\le N} \left( \mathcal{B} ^{\pi} \right) ^{n'-1}\mathcal{B} q
\\ &
\ge \max_{\pi \in \widehat{\Pi }} \left( \mathcal{B} ^{\pi} \right) ^n\mathcal{B} q
\\ &
=\mathcal{B} ^{n+1}q
\end{align*}

Finally, with the results above, we have
\begin{align*} & \quad
\left\| \mathcal{G} ^{N}_{\widehat{\Pi }} q-Q^* \right\| 
\\ &
\le \left\| \mathcal{B} ^{n+1}q-Q^* \right\| 
\\ &
=\left\| \mathcal{B} ^{n+1}q-\mathcal{B} ^{n+1}Q^* \right\| 
\\ &
\le \gamma ^{n+1}\left\| q-Q^* \right\| 
\end{align*}

\end{proof}

\section{Method}\label{sec_method}



In this section, we first present how to conduct \greedyStepValueIteration/ in model-based seeting and then show how to extend to model-free setting.

\textbf{Model-based RL.} Without loss of generality, we present iteration method with state value function $V$.
$$
V_{k+1}\left( s_0 \right) =\max_{a_0} \mathbb{E} _{s_1}\left[ \max_{\pi \in \widehat{\Pi }} \max_{1\le n\le N} \mathbb{E} _{\tau ^{1:n}_{s_1}\sim \pi}\left[ \sum_{t=0}^{n-1}{\gamma ^t}r\left( s_t,a_t \right) +\gamma ^nV_k\left( s_n \right) \right] \right] 
$$
With environment models, the value functions can be updated by
$$
\mathbf{V}_{k+1}=\max_a \max_{\pi \in \widehat{\Pi }} \max_{1\le n\le N} \left[ \mathbf{R}^a+\sum_{i=0}^{n-2}{\gamma \left( \gamma \mathbf{T}^{\pi} \right) ^i\mathbf{R}^{\pi}}+\gamma \mathbf{T}^a\left( \gamma \mathbf{T}^{\pi} \right) ^{n-1}\mathbf{V}_k \right] 
$$
where 
$\mathbf{V}_k$ is a $\left| \mathcal{S} \right|\times 1$ column vector value function;
$\mathbf{R}^a, \mathbf{R}^\pi$ is a $\left| \mathcal{S} \right| \times 1$ column vector of rewards for action $a$ or policy $\pi$;
$\mathbf{T}^a, \mathbf{T}^{\pi}$ is a $\left| \mathcal{S} \right|\times \left| \mathcal{S} \right|$ matrix of transition probability for action $a$ or policy $\pi$.
Note that the matrix 
$
\sum_{i=0}^{n-2}{\gamma \left( \gamma \mathbf{T}^{\pi} \right) ^i},\gamma \mathbf{T}^a\left( \gamma \mathbf{T}^{\pi} \right) ^{n-1}
$
are fixed during iteration process, which means that it can be computed in advance.
Therefore, the computation complexity of each iteration is 
$\mathcal{O}( \left| \mathcal{A} \right| | \widehat{\Pi } |N\left| \mathcal{S} \right|^2 )$ .

\textbf{Model-free RL.}
Note that all the theorems of the new properties hold for any \policySetTextMath/. This means that we can learn with any off-policy data without additional correction.
The difficulty is that it requires taking expectation over different horizon $n$ and then taking the maximum one.
However, taking expectation can be avoided in deterministic MDP.
We first define a generalized operator $\widehat{\mathcal{G} }^{N}_{\widehat{\Pi }}$ for deterministic MDP,
\begin{equation}\label{eq_operator_deterministic_MDP}
\widehat{\mathcal{G} }^{N}_{\widehat{\Pi }} Q\left( s_0,a_0 \right) \triangleq \max_{\pi \in \widehat{\Pi }} \max_{\tau _{s_0,a_0}^{1:N}\sim \pi} \max_{1\le n\le N} \left[ \sum_{t=0}^{n-1}{\gamma ^t}r\left( s_{t}^{},a_{t}^{} \right) +\gamma ^n\max_{a_{n}'} Q\left( s_{n}^{},a_{n}' \right) \right] .
\end{equation}
As can be seen, it does not require making expectation.
Instead, it takes the maximum one over different horizon along a trajectory, and then search over all possible trajectories.
\begin{theorem}
Under deterministic MDP, the operator $\widehat{\mathcal{G} }_{N}^{\widehat{\Pi }}$ is also a contraction on the optimal value function $Q^*$.
\end{theorem}

For discrete MDP which using a tabular Q value function, the value function is updated by
$$
Q_{k+1}\left( s_t,a_t \right) =\max_{\tau _{s_t,a_t}\in \mathcal{D} _{s_t,a_t}} \max_{1\le n\le N} \left[ \sum_{i=0}^{n-1}{\gamma ^i}r_{t+i}+\gamma ^n\max_{a'_{t+n}} Q_k\left( s_{t+n}^{},a'_{t+n} \right) \right] 
$$
where ${\cal D}_{s_t,a_t}$ is the dataset which stores all the trajectories $\traj[]$ starting from $s_t,a_t$.
The algorithm, \emph{\greedyStepQLearning/}, is presented in Appendix \ref{app-sec_algorithm} \Cref{app-alg_GM_QLearning}.

{For continuous MDP}, the value function $Q_{\theta}(s,a)$ is parametrized by the deep neural network (DNN). The parameter $\theta$ is optimized through the following loss function:
$$
	L\left( \theta \right) =\widehat{\mathbb{E} }\left[ Q_{\theta}\left( s_t,a_t \right) -\max_{\tau _{s_t,a_t}\in \mathcal{D} _{s_t,a_t}} \max_{1\le n\le N} \left( \sum_{i=0}^{n-1}{\gamma ^i}r_{t+i}+\gamma ^n\max_{a'_{t+n}} Q_k\left( s_{t+n}^{},a'_{t+n} \right) \right) \right] ^2
$$
where $Q_{\theta'}$ is the target network parametrized by $\theta'$, which is copied from $\theta$ occasionally.
The resulted algorithm is named \emph{\greedyStepDQN/}, presented in APPENDIX \ref{app-sec_algorithm} \Cref{app-alg_GM_DQN}.

Note that our \greedyStepQLearning/ and \greedyStepDQN/ can employ any trajectory data $\traj[][]$ collected by arbitrary policy $\pi$.
This is because \greedyStepOperatorText/ can converge with any \policySetTextMath/ (as stated in \Cref{sec_theorem}).
\T{Together with multi-step bootstrapping}, we have proposed a novel multi-step off-policy algorithm.
Our new proposed algorithms have several advantages than the existing multi-step off-policy methods.
1) It does \emph{not involve importance sampling.} Only the maximization operation is conducted to approximate the optimal value function.
2) Accordingly, it does \emph{not requires knowing the behavior nor the target policy,} as it does not need to calculate the importance ratio $\pi_{\rm target}(a|s)/\pi_{\rm behavior}(a|s)$.
3) Besides, it does not \emph{impose any restriction on the behavior policy nor the target policy.} 
For example, existing methods often require the target policy has positive probability on the action \cite{asis2017multi, precup2000eligibility, munos2016safe}, or the target and the behavior policy are sufficiently close to each other \cite{harutyunyan2016q}.

\subsection{Discussion}
We now discuss several components of the new algorithms.

\textbf{Overestimation.}
For continuous MDP with function approximation, the \greedyReturn/ may lead to an overestimation issue as it greedily chooses the maximal returns.
We found our methods performed well with Double DQN in practice.
Fortunately, abundant approaches have been proposed to address this issue, such as Double DQN \cite{hasselt2010double, van2015deep} and Maxmin DQN \cite{Lan2020Maxmin}.
In this paper, we use Maxmin DQN to extend our method, as it is simple and can be easily fitted to our method.
Maxmin DQN setups several target networks and use the minimal one as the estimate, i.e., $Q(s,a)=\min_{m \in \{1,\ldots,M\}} Q_{\theta'_m}(s,a)$, where $\theta'_m$ is the parameter of $m$-th target network. 

\textbf{Hyperparameter. }
The maximal \step/ $N$ is the main hyperparameter of our algorithm, which decide the maximal horizon that the algorithm looks forward in time.
Larger $N$ can allow faster information propagation but tends to suffer from the overestimation issue.
However, such an issue can be well addressed by the overestimation reduction techniques above.
In practice, we set $N=T-t$, which is the horizon between the current timestep until the end of the trajectory.
This means that the algorithm will look forward in time until the end of the game.
Especially in reward delay task which only provides a reward at the end, 
our method can sufficiently exploit its superiority.



\textbf{Computation Complexity.}
\greedyStepDQN/ requires maximization over $N$ values for each of the batch data.
Two measures can be used to highly accelerate the computation process.
First, the \greedyReturn/ at timestep $t$ can be computed iteratively based on the result at the next timestep $t+1$.
\begin{equation}\label{eq_greedyStepDQN_accleration}
\max_{1\le n\le T-t} R_{n}^{Q}(\tau _{s_t,a_t})
=r_t+\gamma \max \left\{ \max_{a_{t+1}^{\prime}} Q(s_{t+1},a_{t+1}^{\prime}),\max_{1\le n\le T-(t+1)} R_{n}^{Q}(\tau _{s_{t+1},a_{t+1}}) \right\} 
\end{equation}
See APPENDIX \ref{app-sec_algorithm} for the proof.
By this equation our method only needs to do maximization between two adjacent values instead of the values over the entire trajectory.
Second, we sample a complete trajectory and use the equation above to compute the result iteratively for each state-action.
In our implementation, our \greedyStepDQN/ requires almost equal training time with DQN.
See APPENDIX \Cref{app-sec_accelerate} for more detail.

%




\section{Related Work}\label{sec_related_work}


Most methods on multi-step RL are mainly conducted under the umbrella of policy evaluation, whose goal is to evaluate the value of a target policy. This goal naturally allows multi-step bootstrapping.
Formally, the underlying operator, \emph{Multi-Step On-Policy Bellman Operator}, is defined as
\begin{equation}\label{eq_multiStepOnPolicyOperator}
\begin{aligned}
&(\multiStepOnPolicyOperatorMath Q)(s,a) \triangleq 
	\EE_{\traj  } 
	& \Bigg[
		\left.
		\sum_{n=0}^{N-1} \gamma^n r_{t+n} + \gamma^N  Q( s_{t+N}, a_{t+N} )
		\right|
		\parbox{0.46in}{
		\noindent
		$s_t=s$\\$a_t=a$\\$\policySet,N,\pi$
		}
	\Bigg]
\end{aligned}
\end{equation}
The related implementations include multi-step SARSA \citep{sutton2018reinforcement}, Tree Backup \citep{precup2000eligibility}, Q($\sigma$) \citep{asis2017multi}, and Monte Carlo methods (can be regarded as $\infty$-step SARSA).
Such $N$-step bootstrapping can be mixed through exponentially-decay weights as $N$ increases, known as $\lambda$-returns \citep{sutton2018reinforcement,Schulman2016HighDimensional,WhiteW16a}.
The corresponding operator is defined as $ \multiStepOnPolicyOperatorMath[\lambda] \triangleq (1-\lambda) \sum_{ N= 1 }^\infty { \lambda ^N }  \multiStepOnPolicyOperatorMath  $.
\citeauthor{SharmaRJR17} proposed a generalization of $\lambda$-returns, named weighted-returns \cite{SharmaRJR17}. This method assigns a weight to each \step/ $N$,  $ \multiStepOnPolicyOperatorMath[\mathbf{w}] \triangleq \sum_{ N= 1 }^{\infty} { \mathbf{w}_{N} }  \multiStepOnPolicyOperatorMath[N]  $, where $\sum_{N=1}^\infty \mathbf{w}_{N} = 1 $. 
Our method can also be regarded as a special case of weighted-returns, where the weights $\mathbf{w}$ is adaptively adjusted by 
$
\mathbf{w}_{N} = 
\begin{cases}
1 & \text{if } N=\argmax_{N'} \nstepReturn[N'] \\
0 & \text{otherwise}
\end{cases}
$.
Roughly, the \step/ $N$, the decay factor $\lambda$, and the weights $\mathbf{w}$ represents the prior knowledge or our bias on the \step/, and usually has to be tuned in a case-by-case manner. While our method can adaptively adjust the \step/ according to the quality of the data and the learned value function.

\ifISWORD
\else
\begin{table*}[!b]
		\setlength{\tabcolsep}{0pt} 
		\renewcommand\theadfont{\tiny}
		\renewcommand\theadgape{\Gape[0in]}
		\renewcommand\cellgape{\Gape[0in]}
		\scriptsize
		\centering
			\begin{tabular}{@{}
				+m{1.7in}<{\centering}
				Ym{0.6in}<{\centering}
				Ym{1.1in}<{\centering}
				Ym{0.70in}<{\centering}
				Ym{0.70in}<{\centering}
				Ym{0.70in}<{\centering}
				@{}}
				\toprule
Operator   & Converge to &  {Contraction Rate}  & \thead{{NOT requiring} \\ {off-policy correction}} & \thead{{NOT requiring} \\ {knowing policy}} & \thead{Support adaptively \\ adjusting step size}
					\\
										\noalign{\smallskip}
				\midrule
				One-Step Optimality Operator $\oneStepOperatorMath$ 
				& $\color{blue}{Q^*}$	& $\color{red}{\gamma}$   & \checkmarkcolored  & \checkmarkcolored	& \notcheckmarkcolored	\\
				Multi-Step Optimality Operator $\multiStepOperatorMath$ 
				& $\color{red}{Q^*_{ \multiStepOperatorMath }(\leq Q^*)}$  & $\color{blue}\gamma^N$ & \checkmarkcolored  & \checkmarkcolored & \notcheckmarkcolored 	  	 \\
				Multi-Step On/Off-Policy Operator 
				& $\color{red}{Q^\pi}$  & $\color{blue}\gamma^N$ & \notcheckmarkcolored  & \notcheckmarkcolored & \notcheckmarkcolored   	\\
				\textbf{\greedyStepOperatorText/} $\greedyStepOperatorMath$ 
				& $\color{blue}{Q^*}$  &   	\color{blue}{$\gamma$ ($\gamma^2$, $\gamma^N$ under some conditions)} & \checkmarkcolored & \checkmarkcolored & \checkmarkcolored	\\
				\bottomrule
			\end{tabular}
			\caption{
			Properties of the operators.
			}\label{table_summary}
	\end{table*}
\fi

Although these policy evaluation-based methods can easily conduct multi-step learning, they usually require additional correction operation for off-policy data collected from other behavior policies. It is classical to use importance sampling (IS) correction,
	\begin{small}
	\begin{equation}\label{eq_multiStepOffPolicyOperator}
	\begin{aligned}
	& (\multiStepOffPolicyOperatorMath Q)(s,a) \triangleq 
		\EE_{ \pi'\sim \policyDistMath, \traj[\pi']  } 
		& \Bigg[
			\left.
			\sum_{n=0}^{N-1} \gamma^n \zeta_{t+1}^{t+n}  r_{t+n} + \gamma^N \zeta_{t+1}^{t+N}   Q( s_{t+N}, a_{t+N} )
			\right|
			\parbox{0.38in}{
			\noindent
			$s_t=s$\\$a_t=a$\\$\policyDistMath,N,\pi$
			}
		\Bigg]
	\end{aligned}
	\end{equation}
	\end{small}
where $\zeta_{t+1}^{t+n} \triangleq \prod _{ t'=t+1 }^{t+n} {  \frac{ \pi(a_{t'} |s_{t'} ) }{ \pi'(a_{t'}|s_{t'}) } } $ is the \emph{importance sampling (IS) ratio}.
The multiple product terms of IS make the value suffer from high variance. Such issue has motivated further variance reduction techniques, e.g., TB($\lambda$) \citep{precup2000eligibility},  Q($\lambda$) \citep{harutyunyan2016q}, Retrace($\lambda$) \citep{munos2016safe}.
However, these methods are often limited with the restriction on the distance between the target and the behavior policy.
Besides, for these methods, we need to know not only the trajectory but also the behavior policy, i.e., the likelihoods of choosing the behavior action of the behavior policy, $\pi'(a_t|s_t)$.
Our work does not need additional correction and can safely utilize arbitrary data.

There also exists works that follow the idea of value iteration.
\citeauthor{efroni2018beyond} proposed methods to search over action space over multiple steps. 
However, this problem is intractable in computation and they take it as a surrogate MDP problem \cite{tomar2020multi}.
While our methods search over the various rollout data of various behavior policies.
\citeauthor{horgan2018distributed} combined the multi-step trajectory data and the estimated value of the most promising action \citep{horgan2018distributed,barth2018distributed}, which can also be regarded as a type of value iteration. Such methods are simple and have shown promising results in practice. However, it is not clear what the value function learns.
We analyze the convergence properties of the underlying operator, \multiStepOperatorTextMath/, which is defined as
$$
\mathcal{B} ^{N}_{\widehat{\Pi }} Q\left( s_0,a_0 \right) \triangleq \max_{\pi \in \widehat{\Pi }} \mathbb{E} _{\tau _{s_0,a_0}\sim \pi}\left[ \sum_{t=0}^{N-1}{\gamma ^t}r\left( s_{t}^{},a_{t}^{} \right) +\gamma ^N\max_{a_{N}'} Q\left( s_{N}^{},a_{N}' \right) \right] 
$$

\begin{lemma}\label{theorem_multiStepOperator_contraction}
For any two vectors $q, q' \in  \R ^{ |\rls||\rla| } $, $N\geq 1$, 
$
	\|  \multiStepOperatorMath q - \multiStepOperatorMath q' \| \leq \gamma^N 	\|   q -  q' \|.
$
\end{lemma}

\begin{theorem}\label{theorem_multiStepOperator_convergence}
(the \ttt{fix point} of \multiStepOperatorTextMath/)
Let $ Q^{*}_{\multiStepOperatorMath} $ denote the fix point of $\multiStepOperatorMath$, i.e., $\multiStepOperatorMath Q^{*}_{\multiStepOperatorMath} = Q^{*}_{\multiStepOperatorMath} $.
For any $N \geq 2$, $Q^*_{\multiStepOperatorMath} \leq Q^* $. 
If for any $s_0 \in \mathcal{S}$ there exists one policy  $\pi \in \widehat \Pi $ such that $\pi(s_{t}) = \pi^*(s_{t})$ for any $s_{t}$ ($0\leq t\leq N$) in $\tau_{s_0} \sim \pi^*$, then $Q^*_{\multiStepOperatorMath} = Q^*$.
\end{theorem}

The theorems above imply that \multiStepOperatorTextMath/ ($N\geq 2$) converges faster than \oneStepOperatorTextMath/ does.
However, it generally converges to a suboptimal value function --- converging to optima only happens when \emph{all} the behavior policies are optimal, which is almost never guaranteed in practice.





In summary, our \greedyStepOperatorTextMath/ can converge to the optimal value function with a rate of $\gamma$ (or $\gamma$, $\gamma^N$ with some condition). It does \emph{not} require off-policy correction and only requires access to the trajectory data while not necessary the behavior policy $\pi$ (i.e., the policy distribution $\pi(\cdot|s)$). Furthermore, it can adaptively adjust the step size by the quality of the trajectory data.
To the best of our knowledge, none of the \T{existing operators own all of these properties together.}
We summarize the properties of all the referred operators in \Cref{table_summary}.

\section{Experiment}\label{sec_experiment}
We designed our experiments to investigate the following questions.
1) Can our methods \greedyStepDQN/ (\greedyStepQLearning/) improve previous methods, DQN (Q learning)? 
How does it compare with the state-of-art algorithm?
2) {What is the effect of the components of our methods, e.g., 
the overestimation reduction technique}?
3) How does our algorithm behave in practice? 
For example, What are the sizes of the bootstrapping steps during the training phase?

We implement our algorithms, \greedyStepQLearning/ and \greedyStepDQN/, by extending Q learning and Maxmin DQN respectively.
For all tasks, we set the maximal \step/ $N=T-t$, which is the length between the current timestep and the end of the trajectory.

We choose \emph{DQN} and \emph{Maxmin DQN} as baselines. We also compare with several state-of-the-art multi-step off-policy methods.
\emph{Multi-Step DQN}: a vanilla multi-step version of Q learning without any off-policy correction \citep{horgan2018distributed, barth2018distributed}.
\emph{Multi-Step SARSA}: a classical one-step on-policy algorithm \citep{sutton2018reinforcement}.
All the proposed methods adopt the same implementations to ensure that the differences are due to the algorithm changes instead of the implementations. 
We reused the hyper-parameters and settings of neural networks in \cite{Lan2020Maxmin}, in which the learning rate was chosen from $[3 \times 10^{-3}; 10^{-3}; 3 \times 10^{-4}; 10^{-4}; 3 \times 10^{-5} ]$.
For our \greedyStepDQN/, the number of target networks was chosen from $[1,2,4,6]$, while for Maxmin DQN it was chose from [2,3,4,5,6,7,8,9] and we use the results reported in \cite{Lan2020Maxmin}.

\subsection{Performance}\label{sec_performance}

\ifISWORD
\else
\begin{figure}[t]
    \centering
       	\def\widthproperty{0.31}
  	   	\centering{
   		\subfloat[Choice]{
   			\label{fig_TheChoice}
			\includegraphics[width=\widthproperty\linewidth]{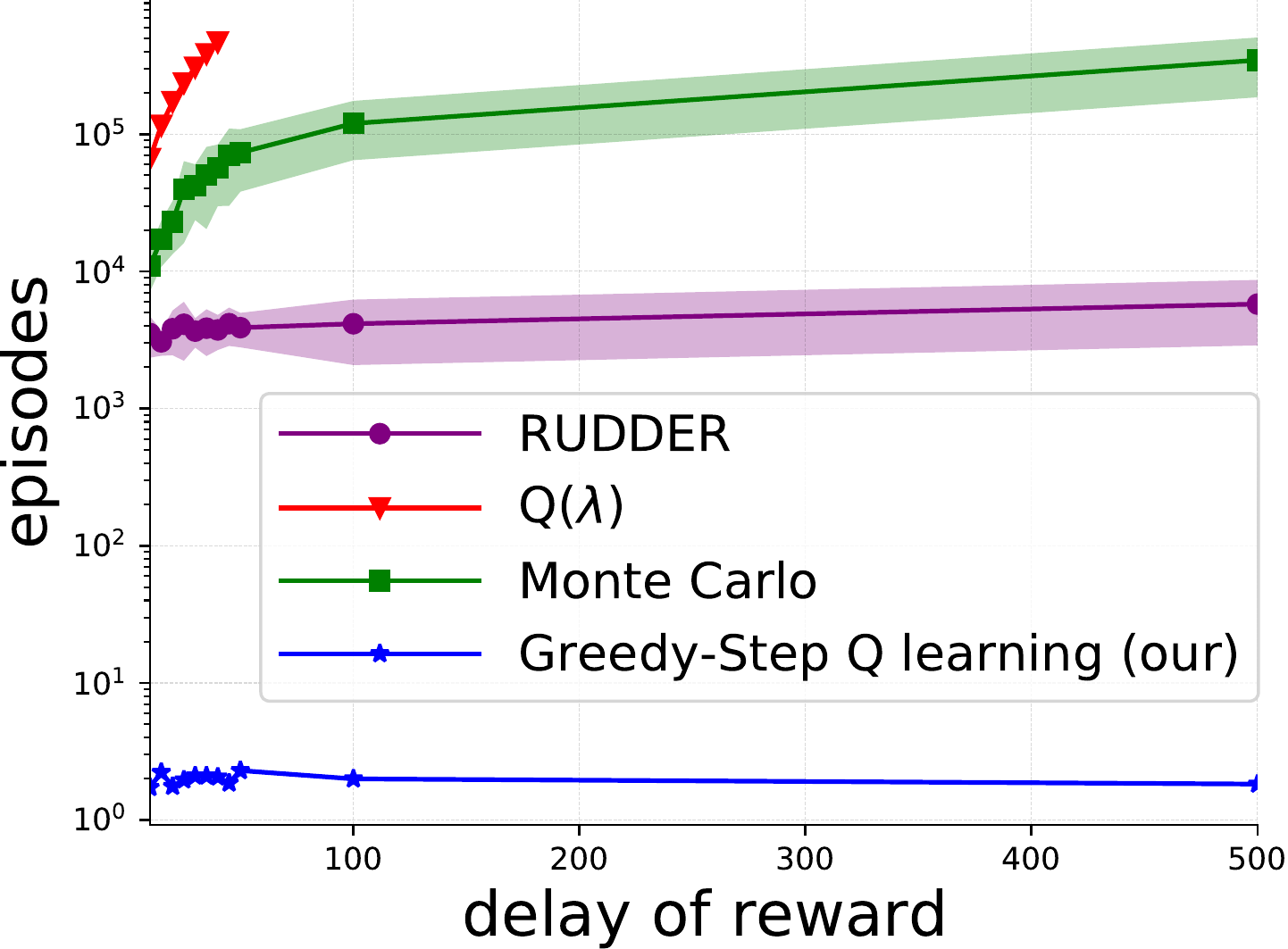}
   		}\hfill
   		\subfloat[Trace Back]{
   			\label{fig_TraceBack}
			\includegraphics[width=\widthproperty\linewidth]{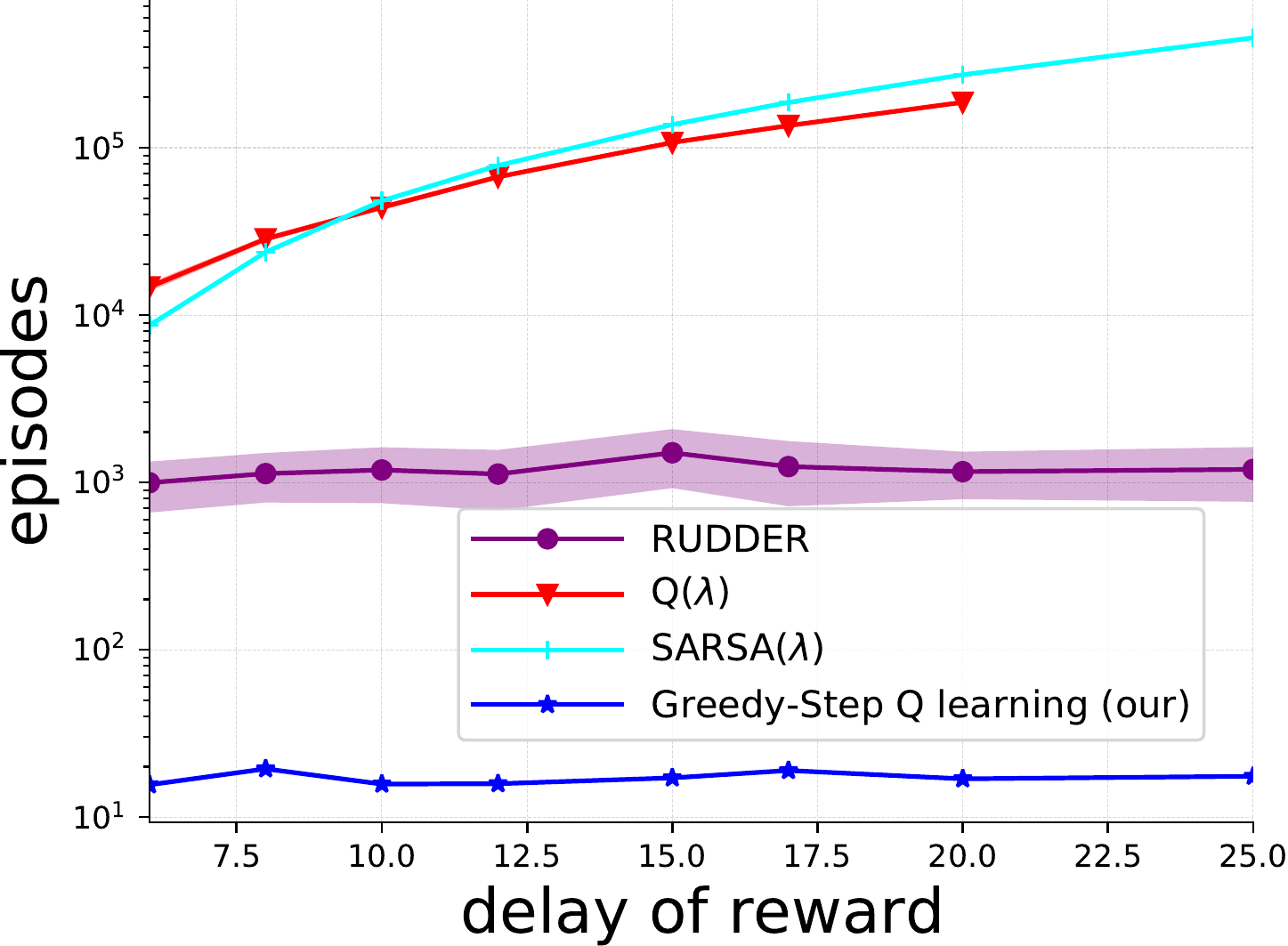}
   		}\hfill
   		   		\subfloat[Grid World]{
   		   			\label{fig_GridWorld_size}
   					\includegraphics[width=\widthproperty\linewidth]{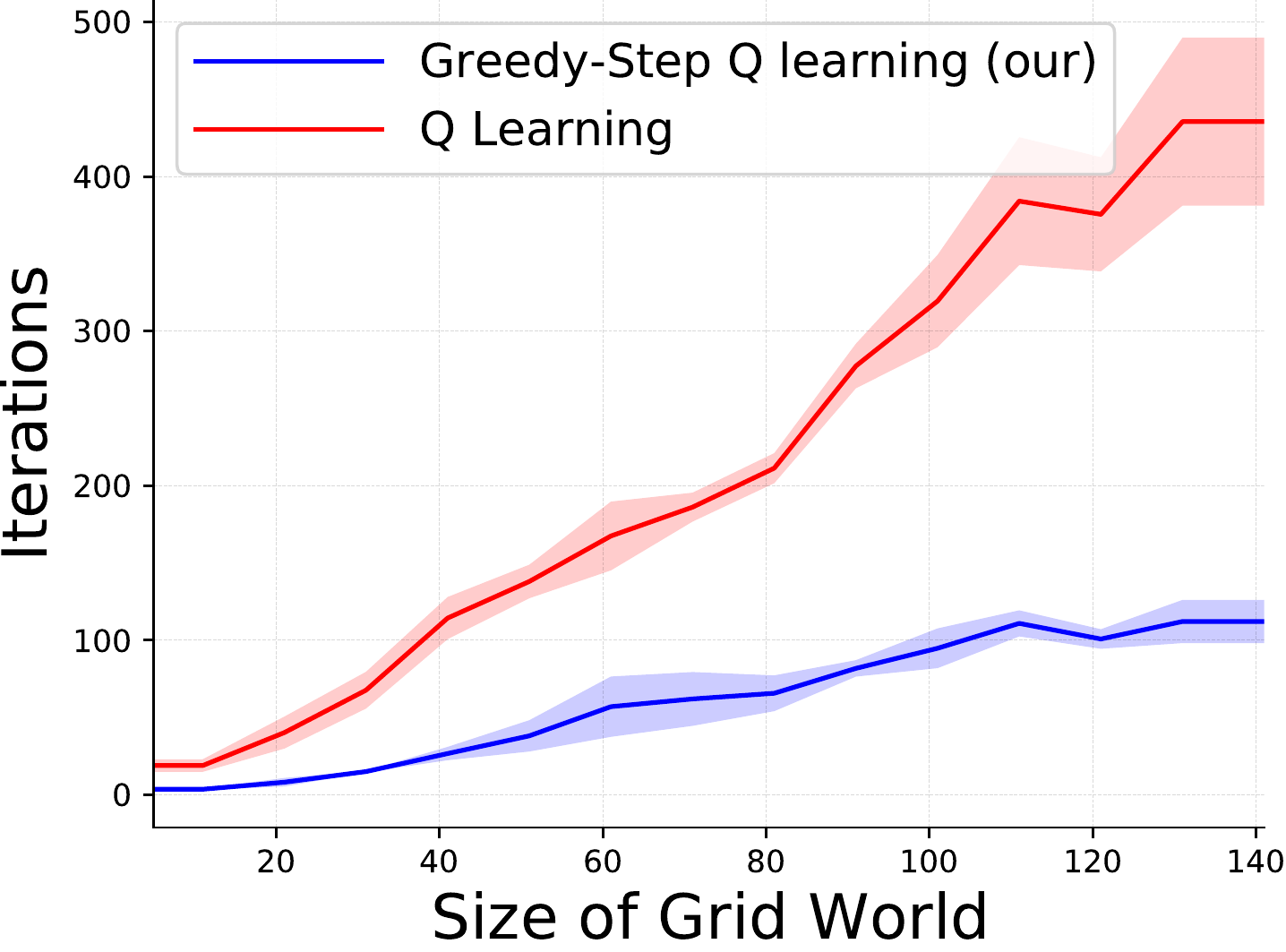}
   		   		}
   		  	   	}
    \caption{
    	Results on three toy examples: Choice, Trace Back, and Grid World.
    }\label{fig_toy_example}
\end{figure}
\fi

\textbf{Toy Tasks.}
We first evaluate algorithms on three toy tasks to understand the behavior of the new algorithms.
The algorithm are run with 100 trials on these tasks.

We first evaluate on two tasks with delayed rewards, provided in \cite{NEURIPS2019_16105fb9}.
Both the environments only provide a reward at the end of the game.
The final reward is associated with the previous actions.
For example, in task ``Trace Back'', only when the agent executes the exact two actions at the first two steps, then will it be provided with a highest score.
The algorithms are evaluated until the task is solved.
We compare our method with traditional eligibility trace methods and RUDDER \cite{NEURIPS2019_16105fb9}.
As shown in \Cref{fig_toy_example} (a) and (b), our method significantly outperforms all the algorithms on both the tasks.
For example, on Trace Back, our method requires only 20 episodes to solve the task, while the best algorithm RUDDER requires more than 1000.
Notably, the cost episode of our \greedyStepQLearning/ do not observably increase as the ``delay'' increases.
In contrast, other methods such as Q($\lambda$) require exponentially increasing episodes.
This is because our algorithm update the values by looking forward in time until the end of the game. Thus the final reward can fast propagate to the previous actions across a long horizon.

\ifISWORD
\else
\begin{figure*}[!b]
    \centering
   	\def\widthproperty{0.33}
	\includegraphics[trim={5cm 1cm 0 0},width=1.\linewidth]{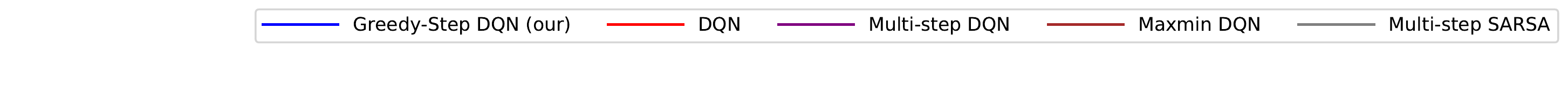}
   	\centerline{
	   	\includegraphics[width=\widthproperty\linewidth]{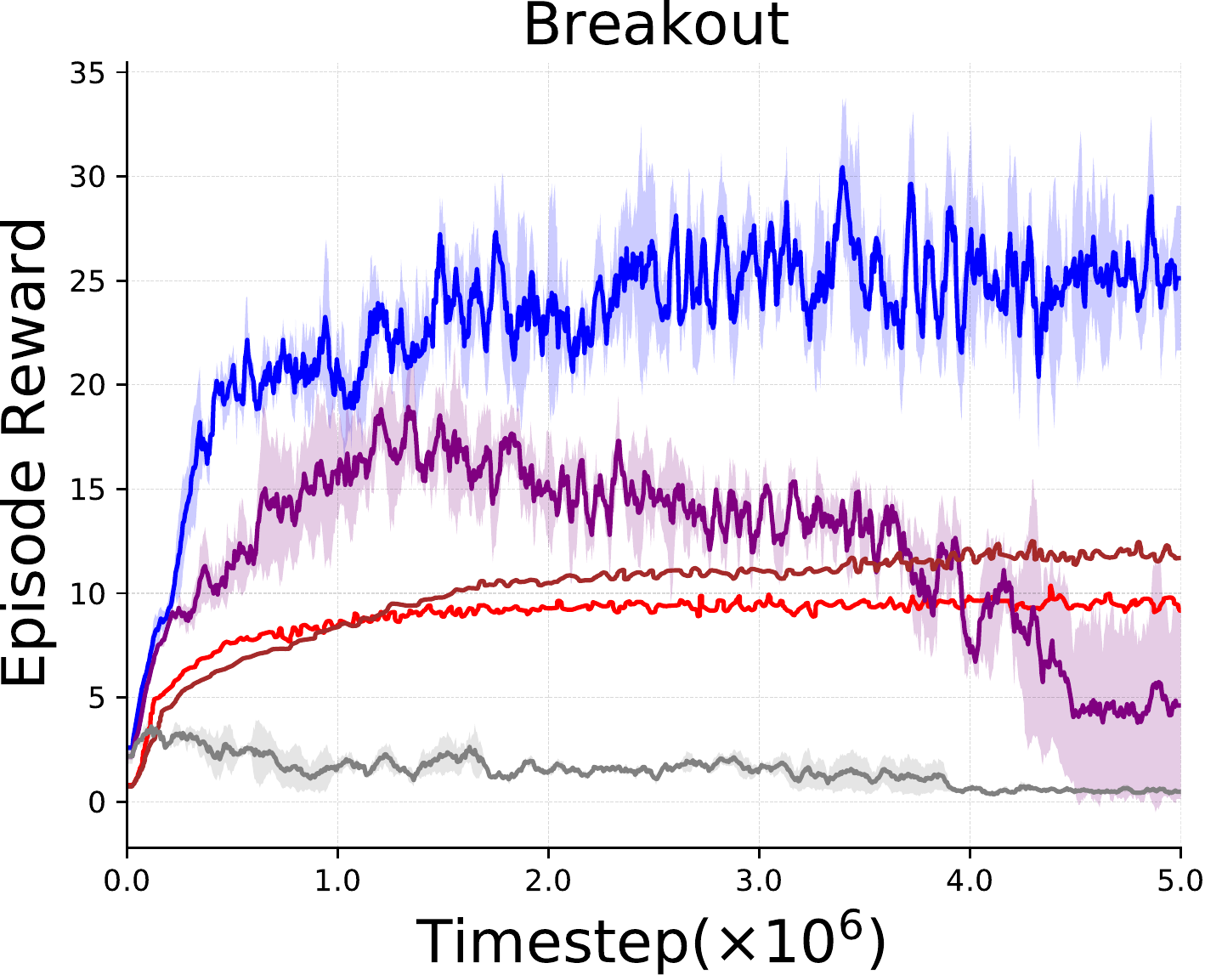}
	   	\includegraphics[width=\widthproperty\linewidth]{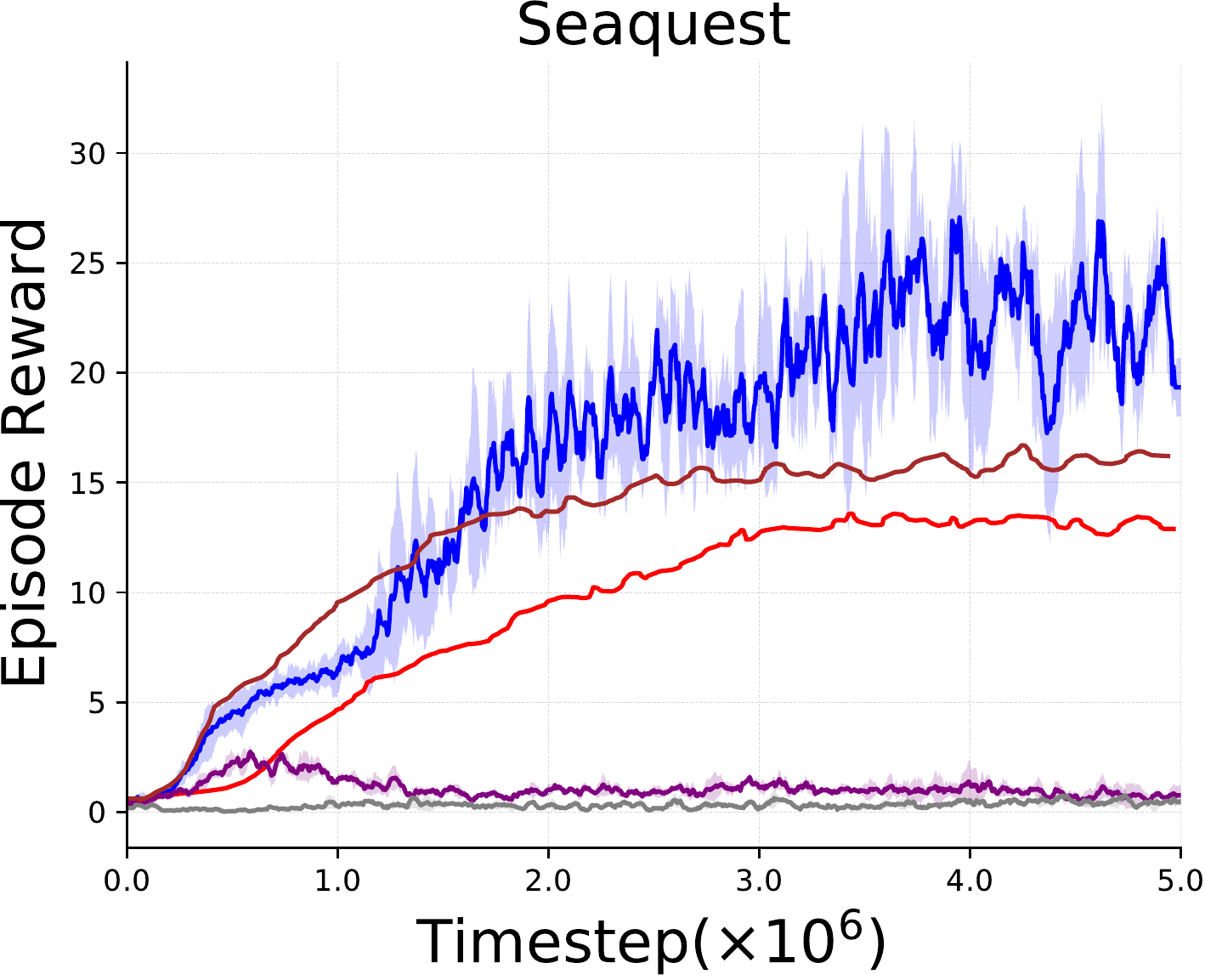}
	   	\includegraphics[width=\widthproperty\linewidth]{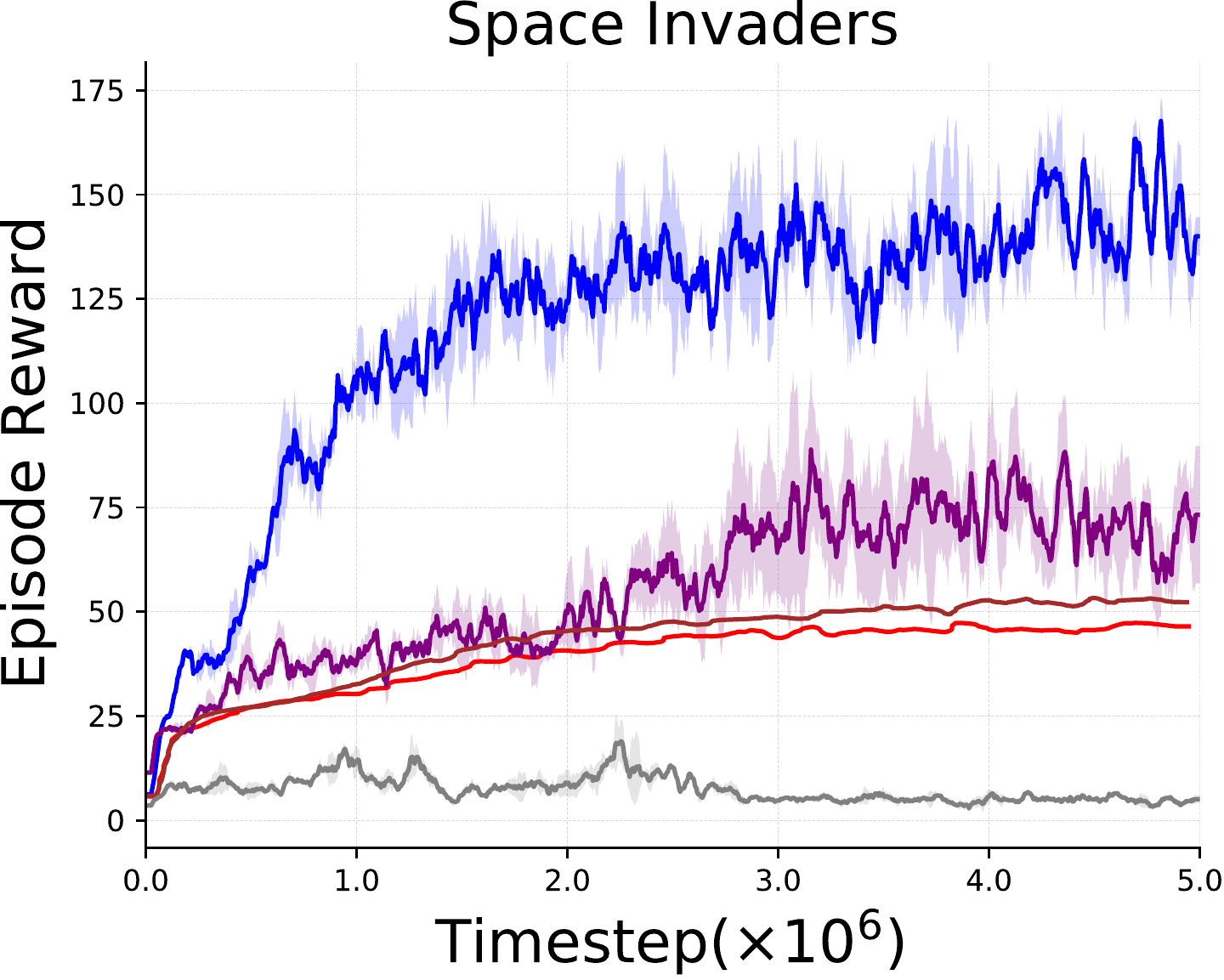}
   	}
  	\centerline{
	   	\includegraphics[width=\widthproperty\linewidth]{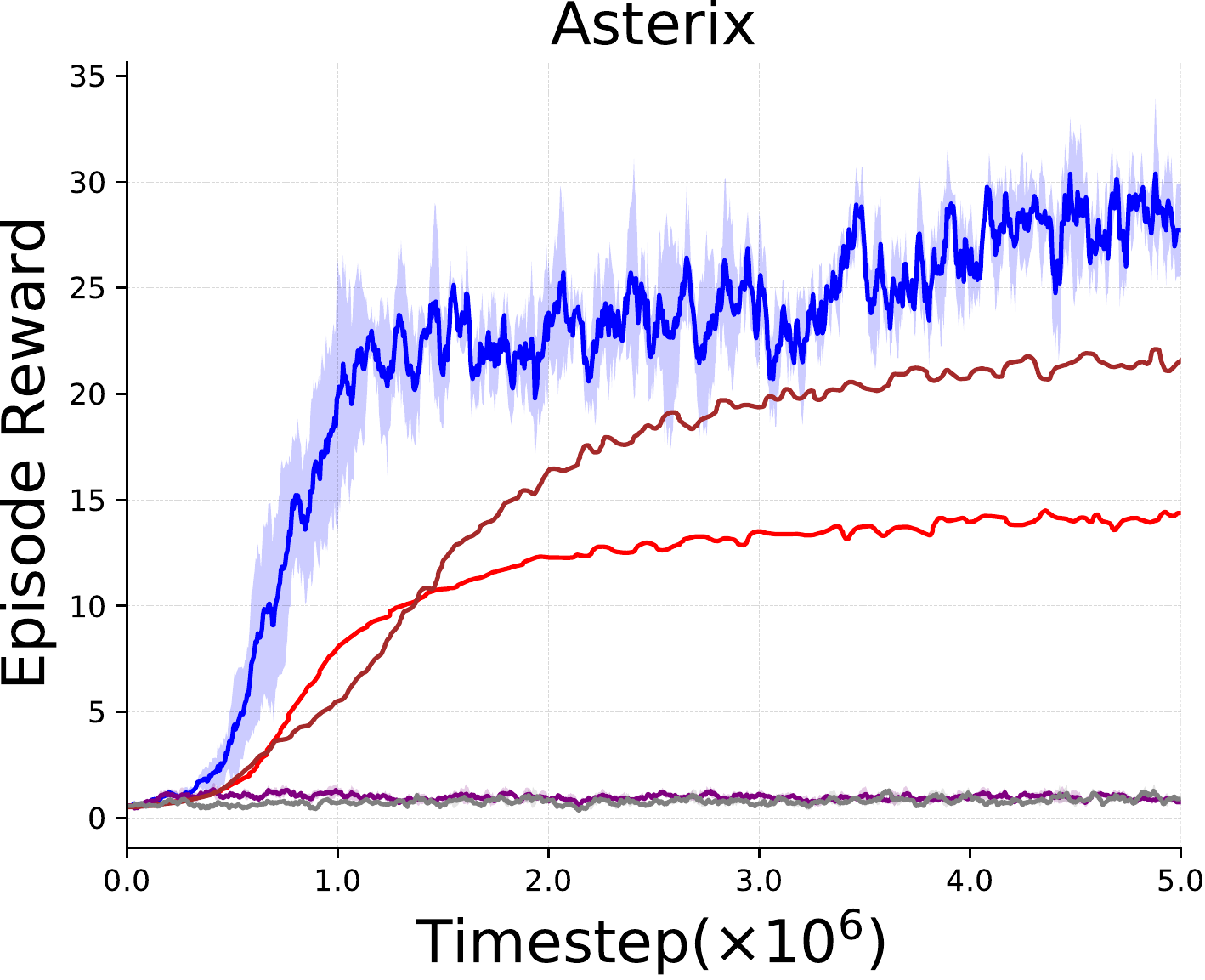}
	   	\includegraphics[width=\widthproperty\linewidth]{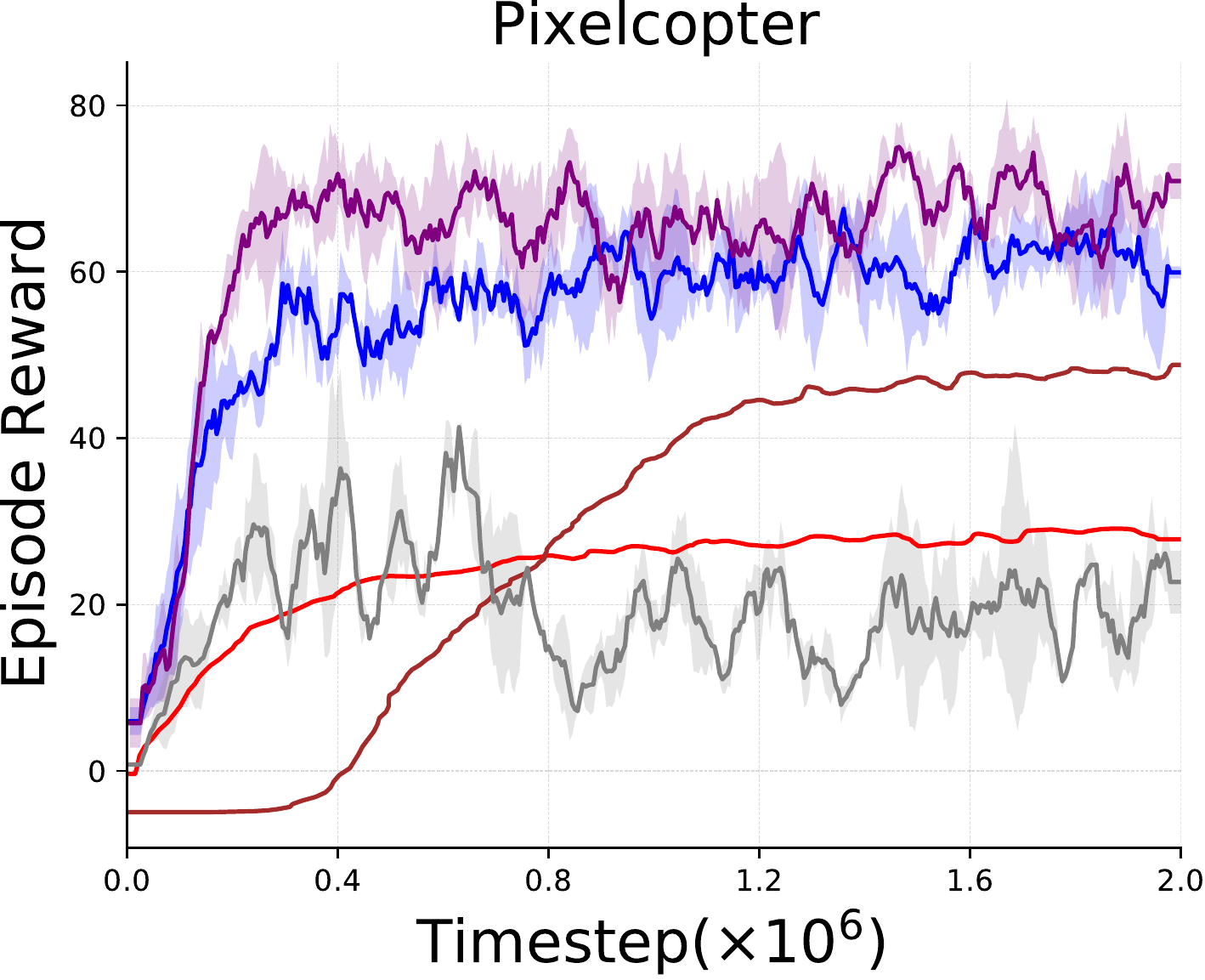}
	   	\includegraphics[width=\widthproperty\linewidth]{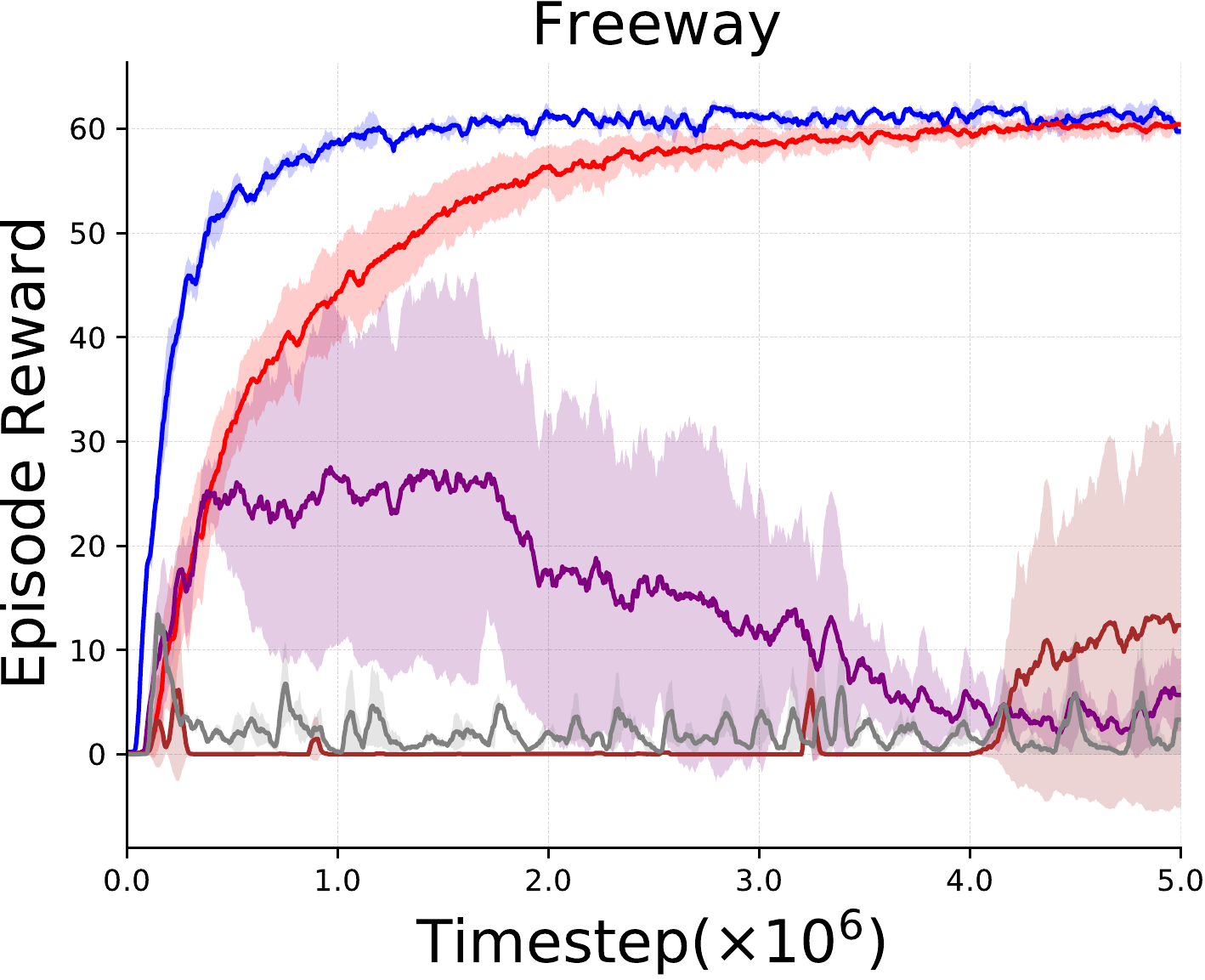}
  	}
    \caption{
 		  Episode rewards of the algorithms during the training process averaged over \nrandomseed/ random seeds.
    }\label{fig_score_minatar_greedy_step}
\end{figure*}
\fi

\textbf{MinAtar Games.}
We then evaluated algorithms on benchmark tasks from Gym \citep{Brockman2016OpenAI}, PyGame Learning Environment (PLE) \citep{tas}, and MinAtar \citep{Kyo}.
Each algorithm was run with \nrandomseed/ random seeds. 
The policies are evaluated occasionally during the training process.

\Cref{fig_score_minatar_greedy_step} shows the performance of the algorithms. 
\greedyStepDQN/ significantly outperforms almost all the compared algorithms in both reward and sample efficiency on almost all the tasks except Pixelcopter.
Especially on tasks like Breakout, Seaquest, and Space Invaders, our \greedyStepDQN/ achieves almost twice the reward of the best of the compared methods.
On tasks like Pong, MountainCar, and Freeway, although several algorithms converge to the same level as our \greedyStepDQN/, these algorithms required more than twice the samples to converge than our \greedyStepDQN/ does.




We then evaluate the algorithms in an offline setting to avoid the impact of exploration.
We collect several trajectory data by random policies in advance. 
The algorithms are trained with these data without additional collection.
We make the experiment on a Grid World task. The agent receives a reward of $-1$ at each step until it reaches the goal.
As can be shown in \Cref{fig_GridWorld_size}, as the size of Grid World increases, the performance gap between our method and Q learning becomes larger.

\subsection{Algorithm Behavior}\label{sec_ablation_study}

%
%

\ifISWORD
\else
\begin{figure}[!h]
    \centering
   	\def\widthproperty{0.25}
   	\centerline{
  	  		   	\includegraphics[trim={0 1cm 0 0},width=1.\linewidth]{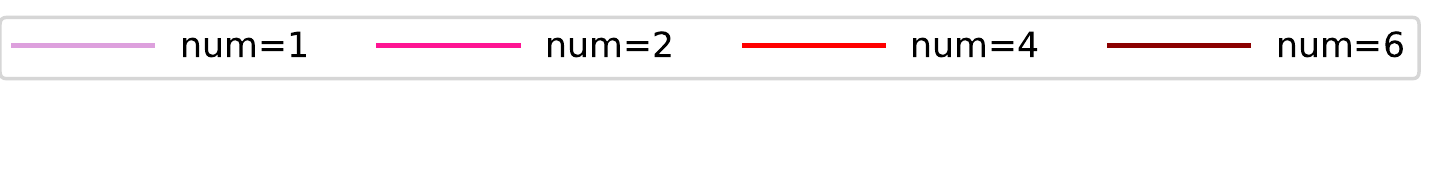}
  	  		   	}
  	   	\centerline{
			\includegraphics[width=\widthproperty\linewidth]{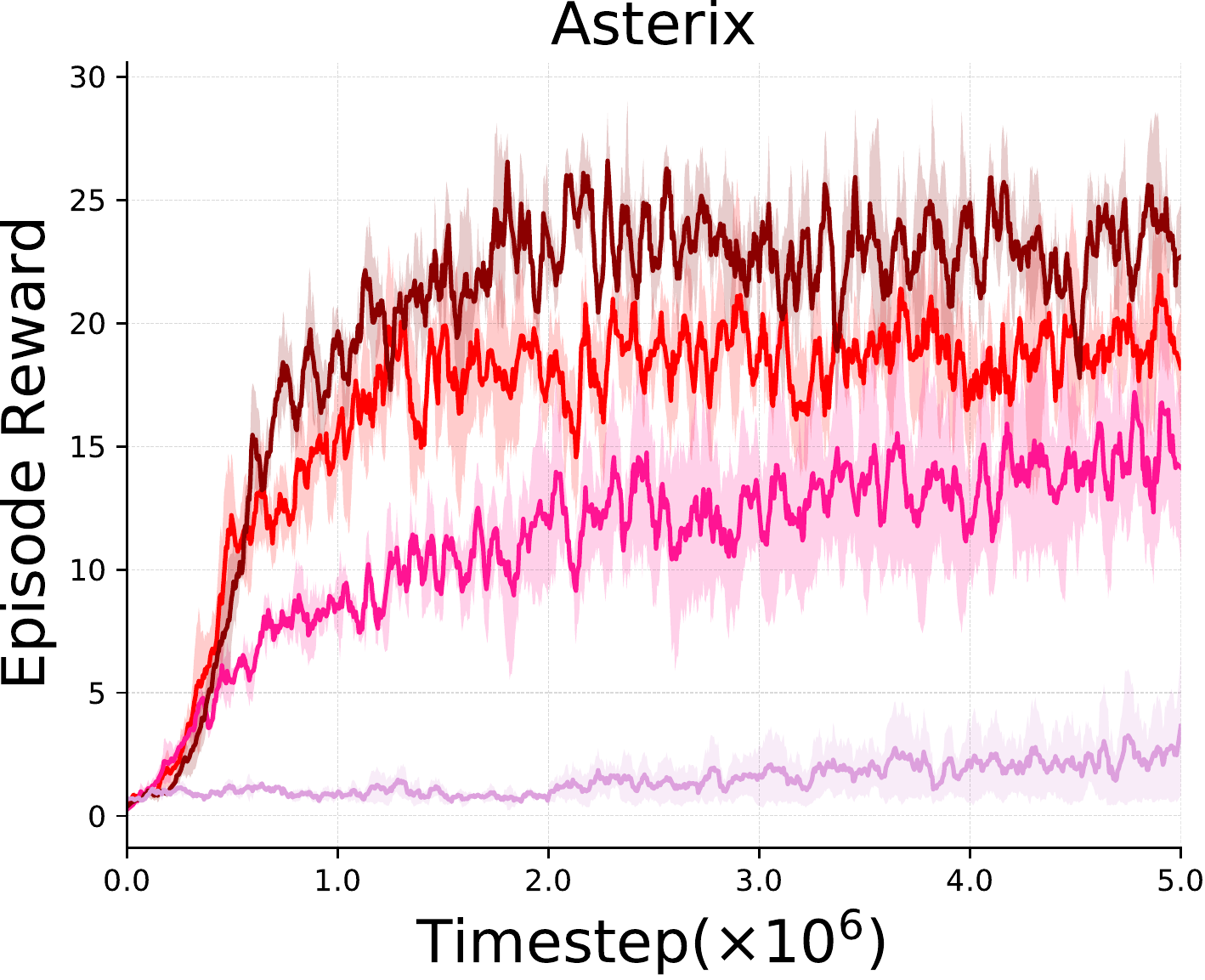}
			\includegraphics[width=\widthproperty\linewidth]{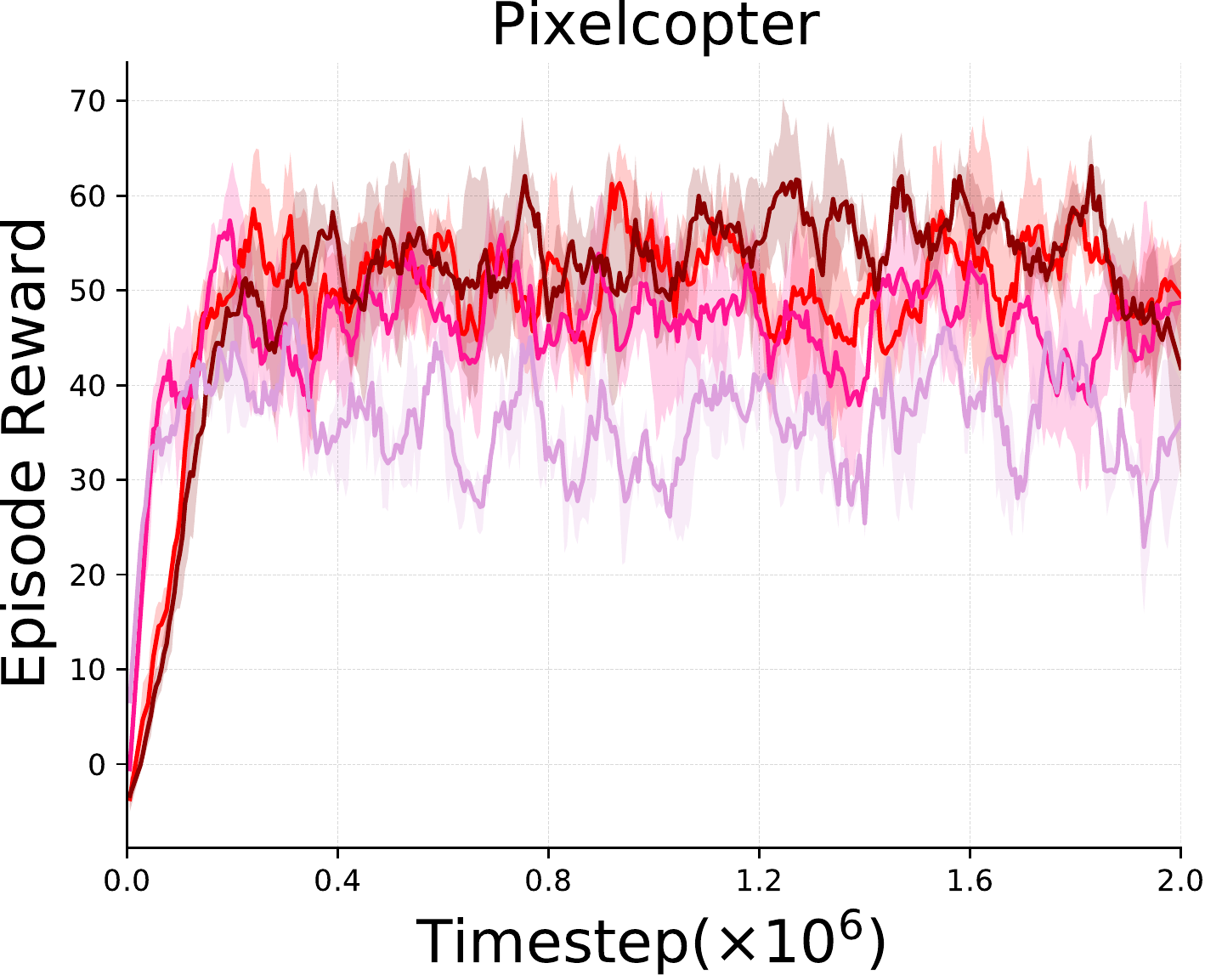}
			\includegraphics[width=\widthproperty\linewidth]{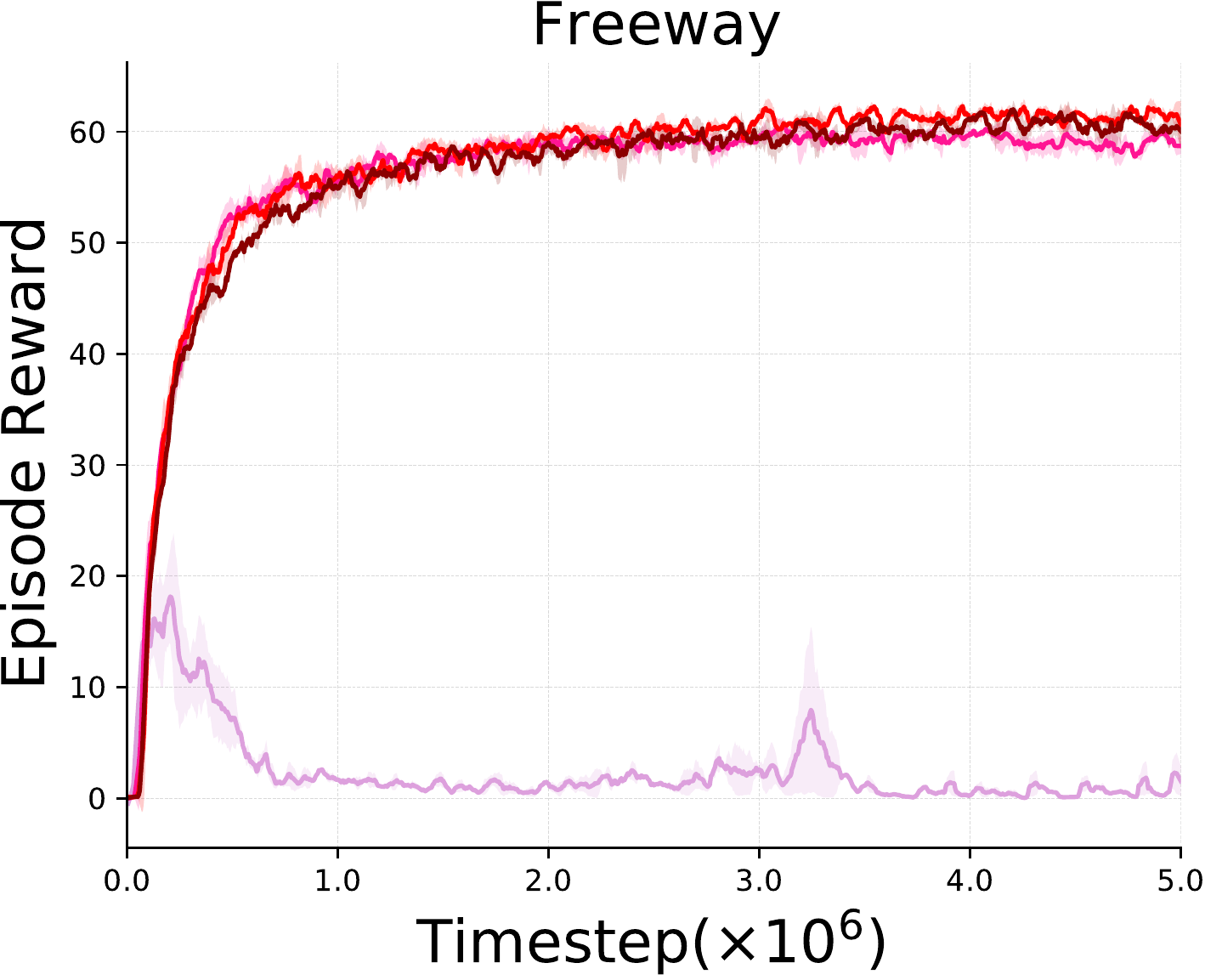}
  	   	}
  	   	\centerline{
			\includegraphics[width=\widthproperty\linewidth]{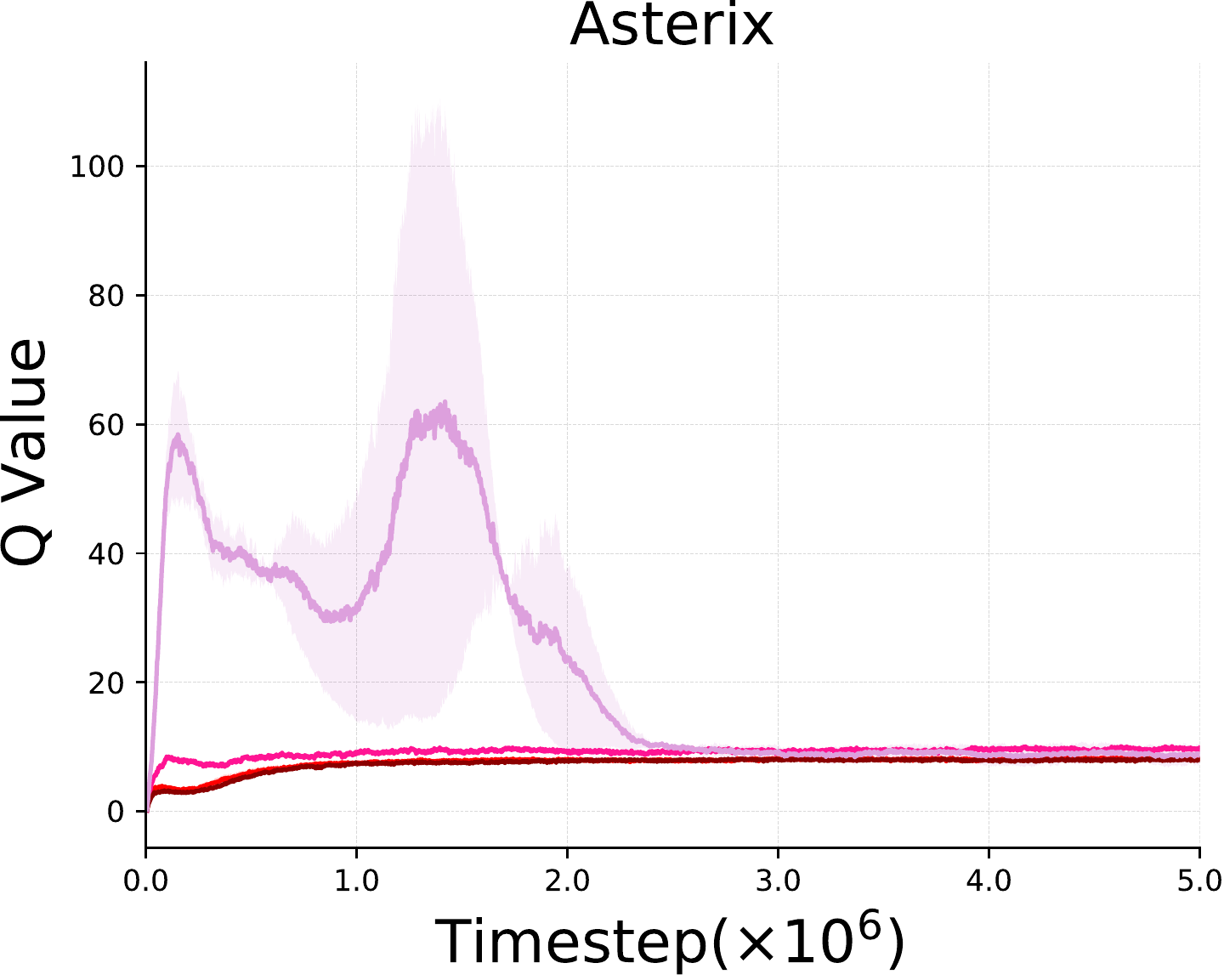}
			\includegraphics[width=\widthproperty\linewidth]{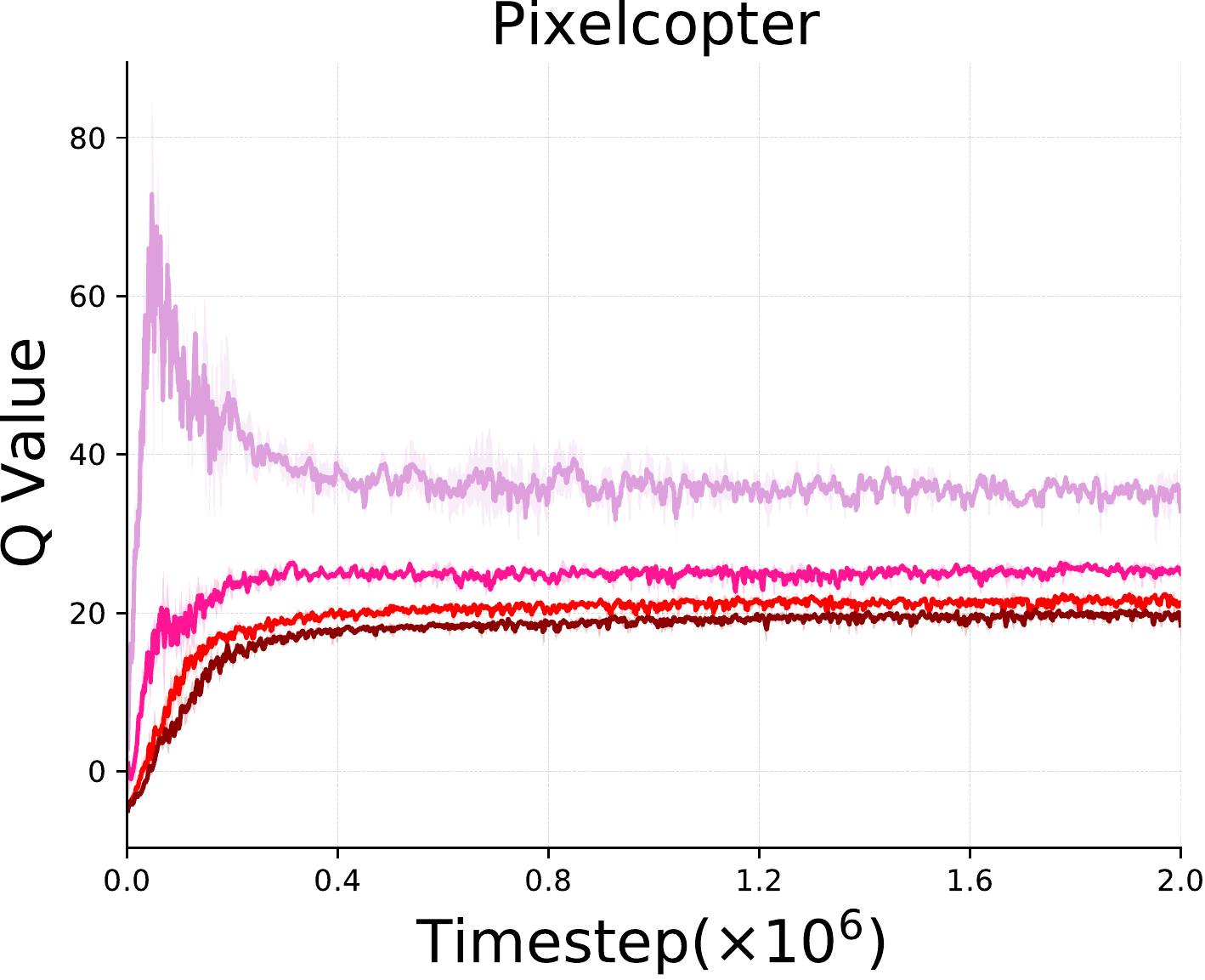}
			\includegraphics[width=\widthproperty\linewidth]{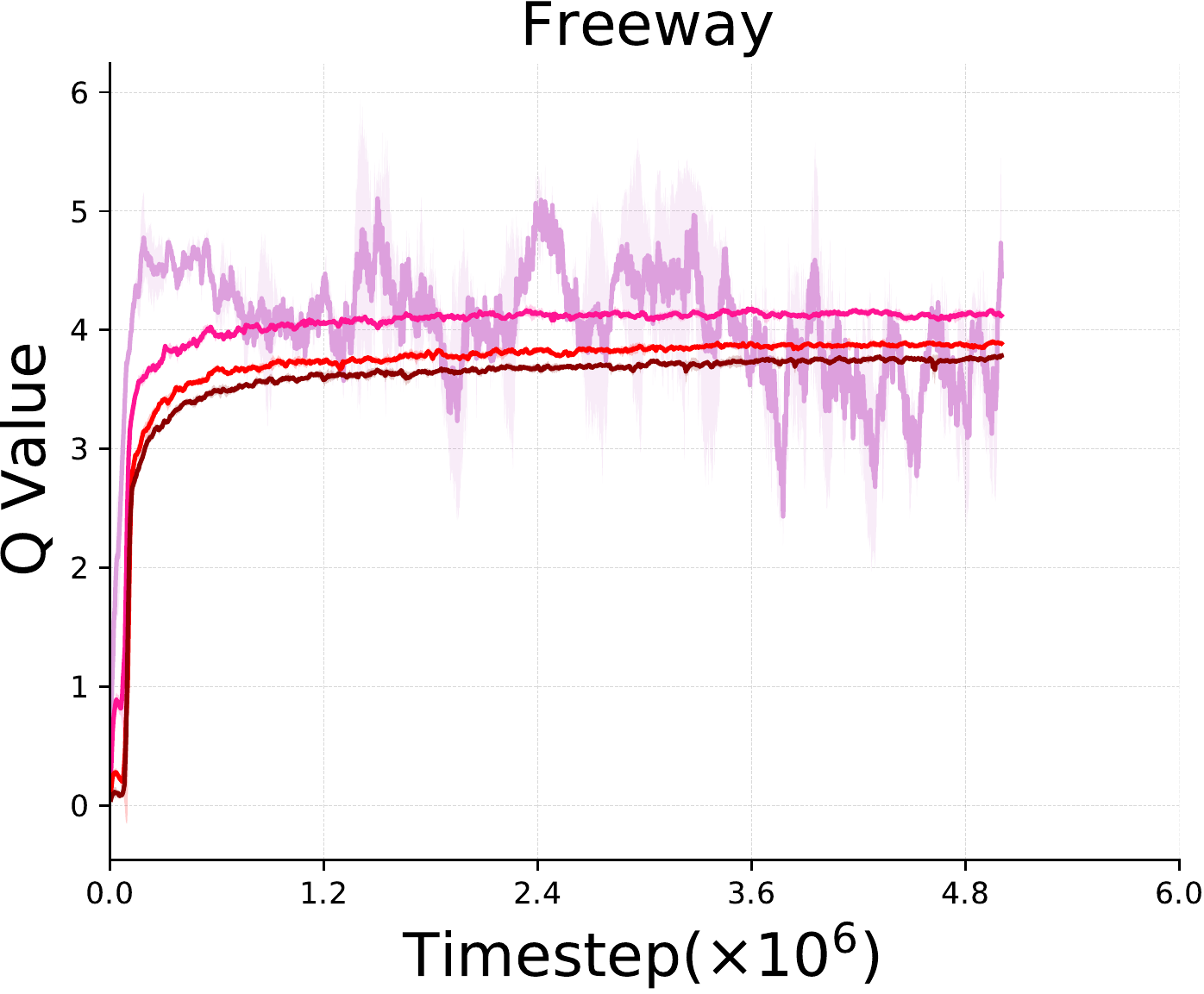}
  	   	}
    \caption{
		(Top) \textbf{The episode rewards} and (Bottom) \textbf{the mean of the estimated $Q$ value} of \greedyStepDQN/ with varying number of target networks during the training process.
    }\label{fig_overestimation}
\end{figure}
\fi

\textbf{Overestimation.}
The maximization over various returns can easily result in an overestimation issue.
One important component to the success of \greedyStepDQN/ is the technique of overestimation reduction.
In this paper, we employ Maxmin DQN, which setups several target networks and use the minimal one as the target.
As shown in \Cref{fig_overestimation}, when the number of target networks is sufficiently large (num=6), the algorithm performs better and more stable, and the Q value is relatively smaller. In general, we found that in practice by setting the number of target networks to be 6, our \greedyStepDQN/ can often obtain a good performance.



\ifISWORD
\else
\begin{figure}[!h]
    \centering
       	\def\widthproperty{0.33}
      	\centerline{
    	   	\includegraphics[width=\widthproperty\linewidth]{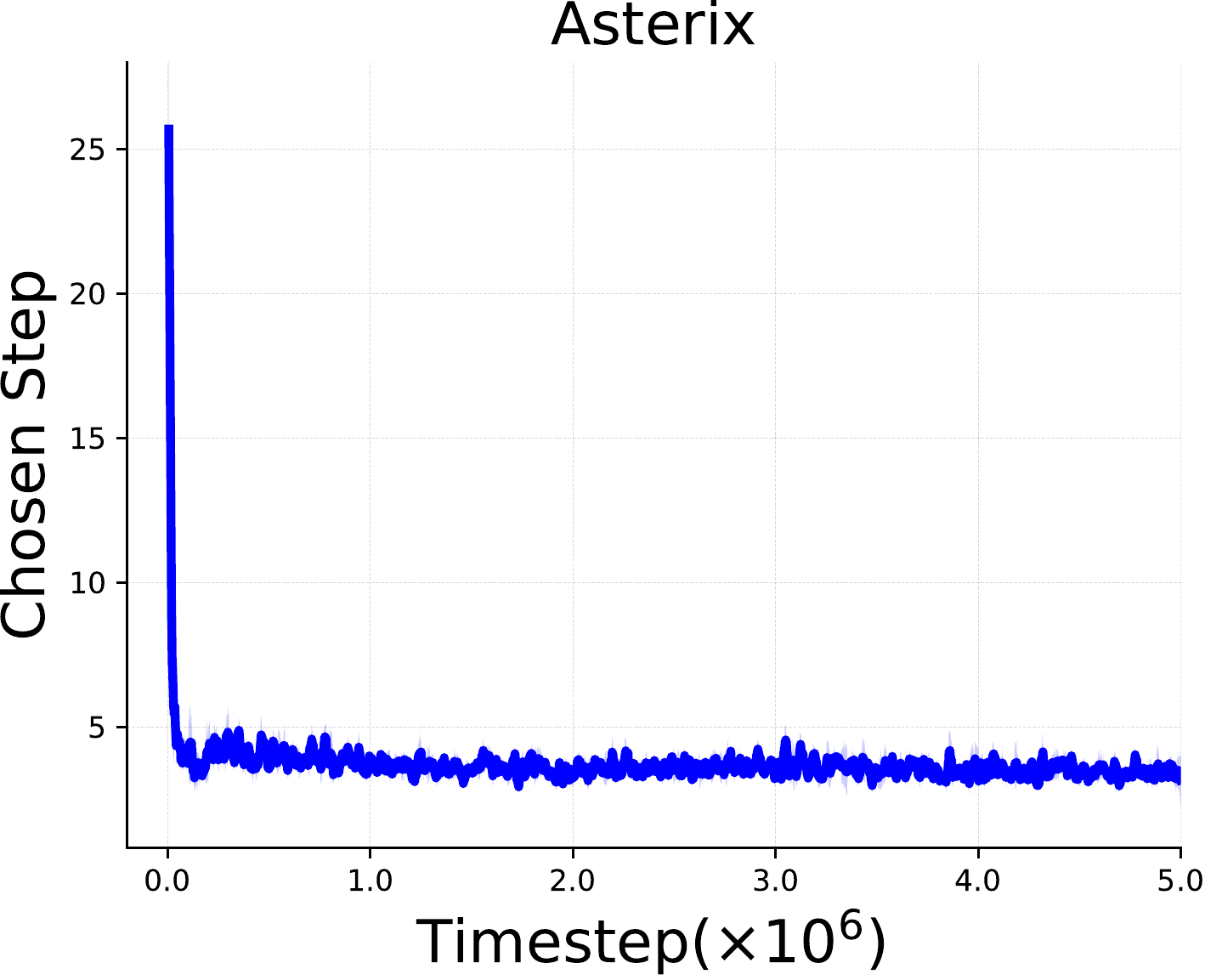}
   	   		\includegraphics[width=\widthproperty\linewidth]{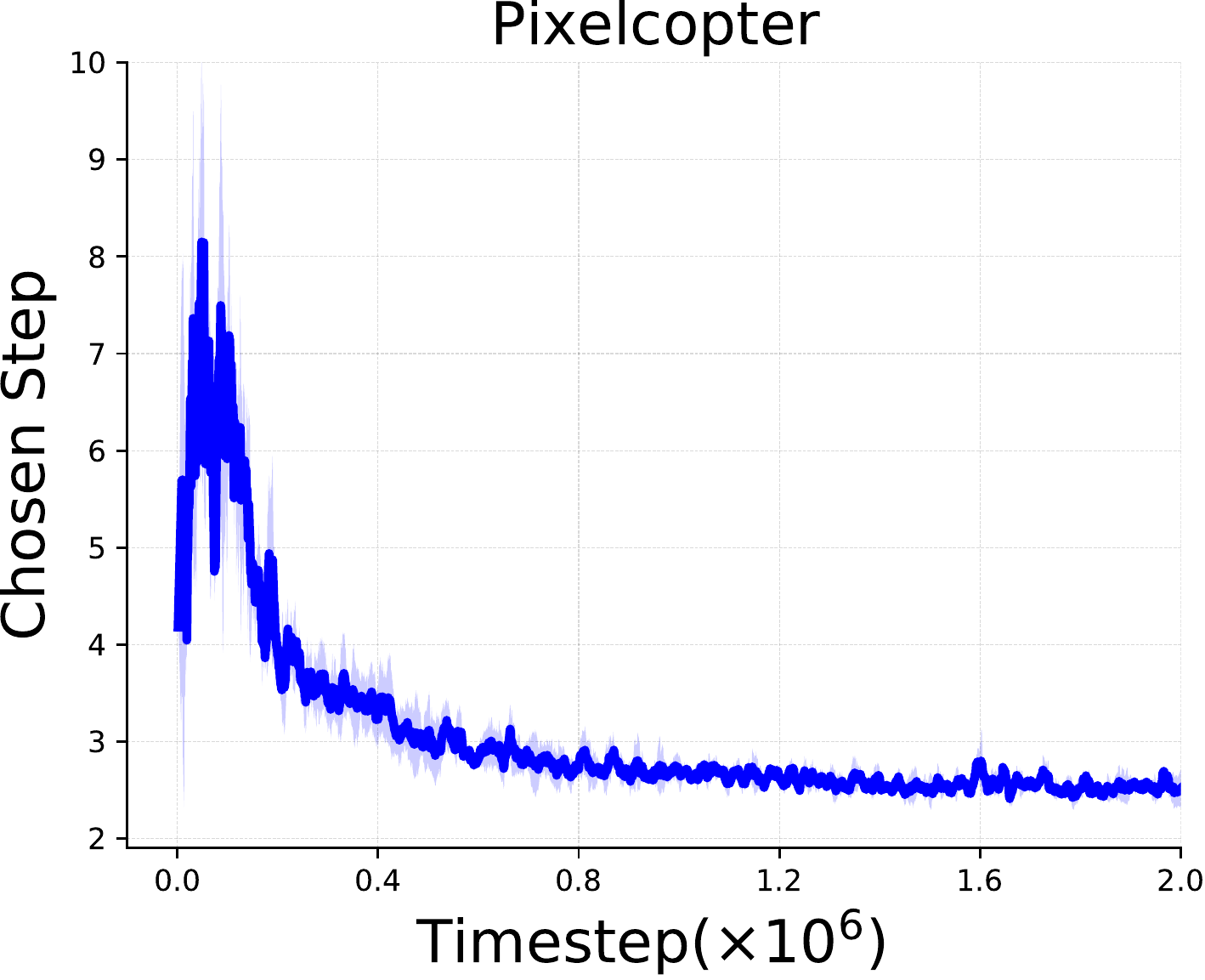}
    	   	\includegraphics[width=\widthproperty\linewidth]{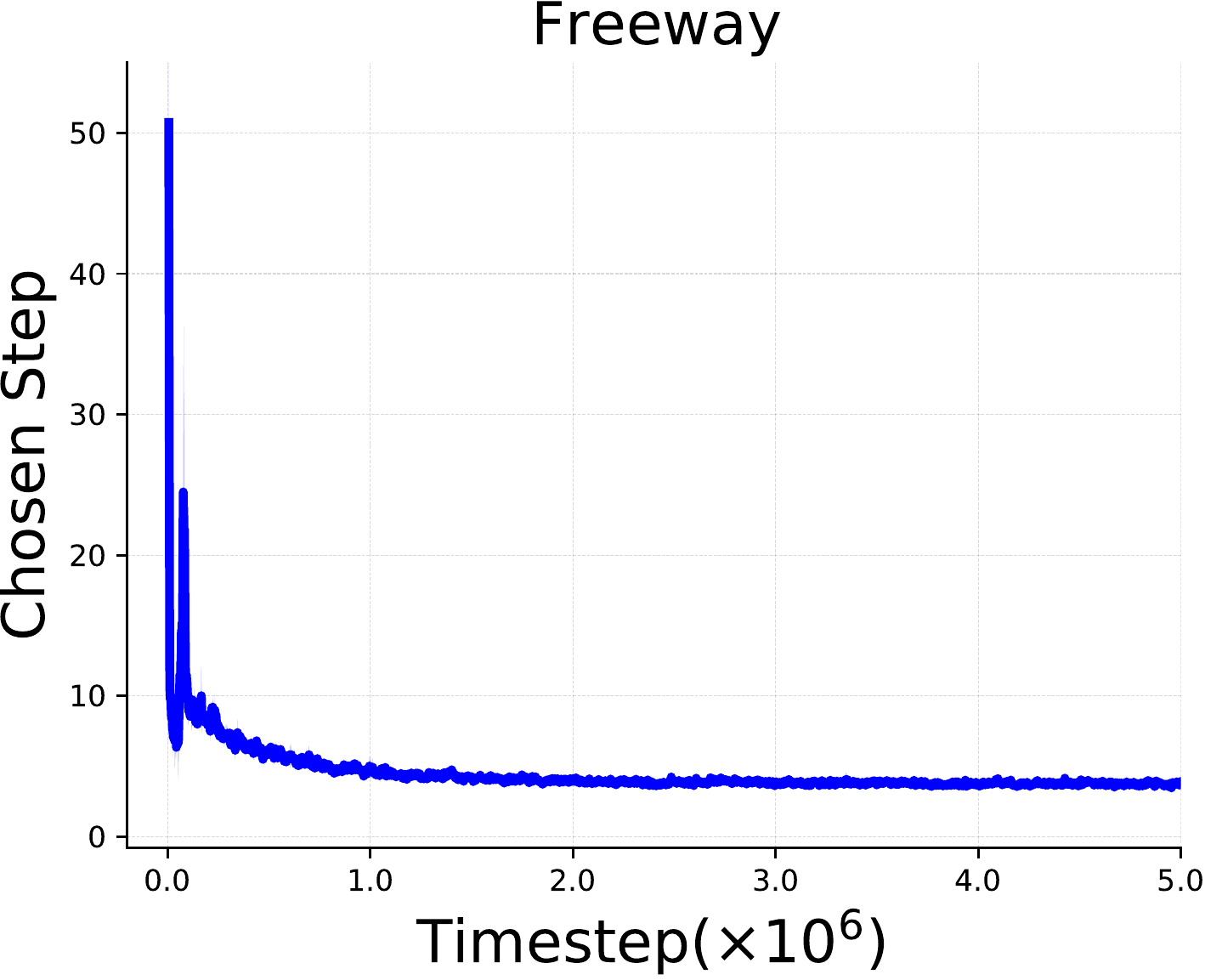}
      	}
      	\centerline{
	    	   	\includegraphics[width=\widthproperty\linewidth]{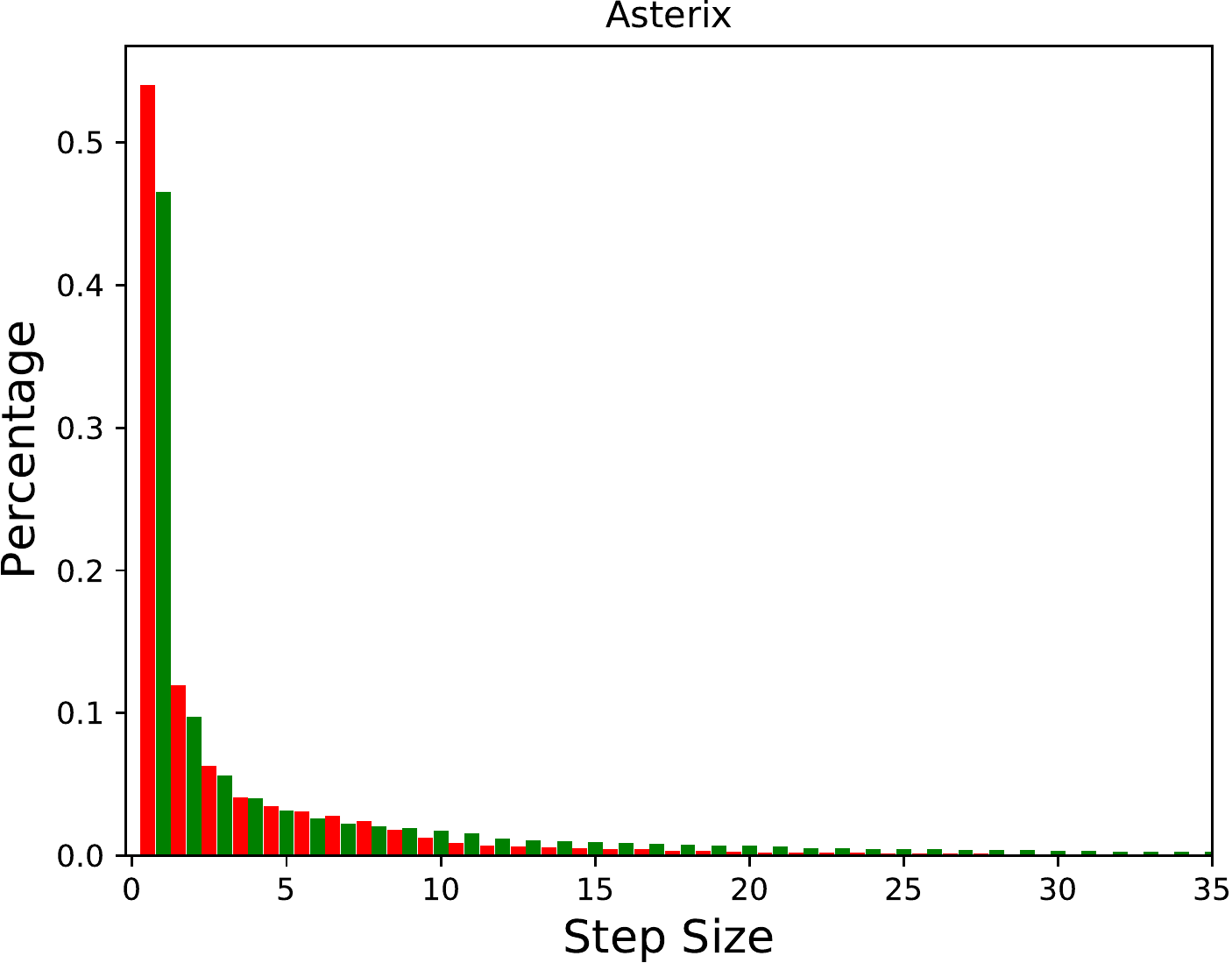}
	   	   		\includegraphics[width=\widthproperty\linewidth]{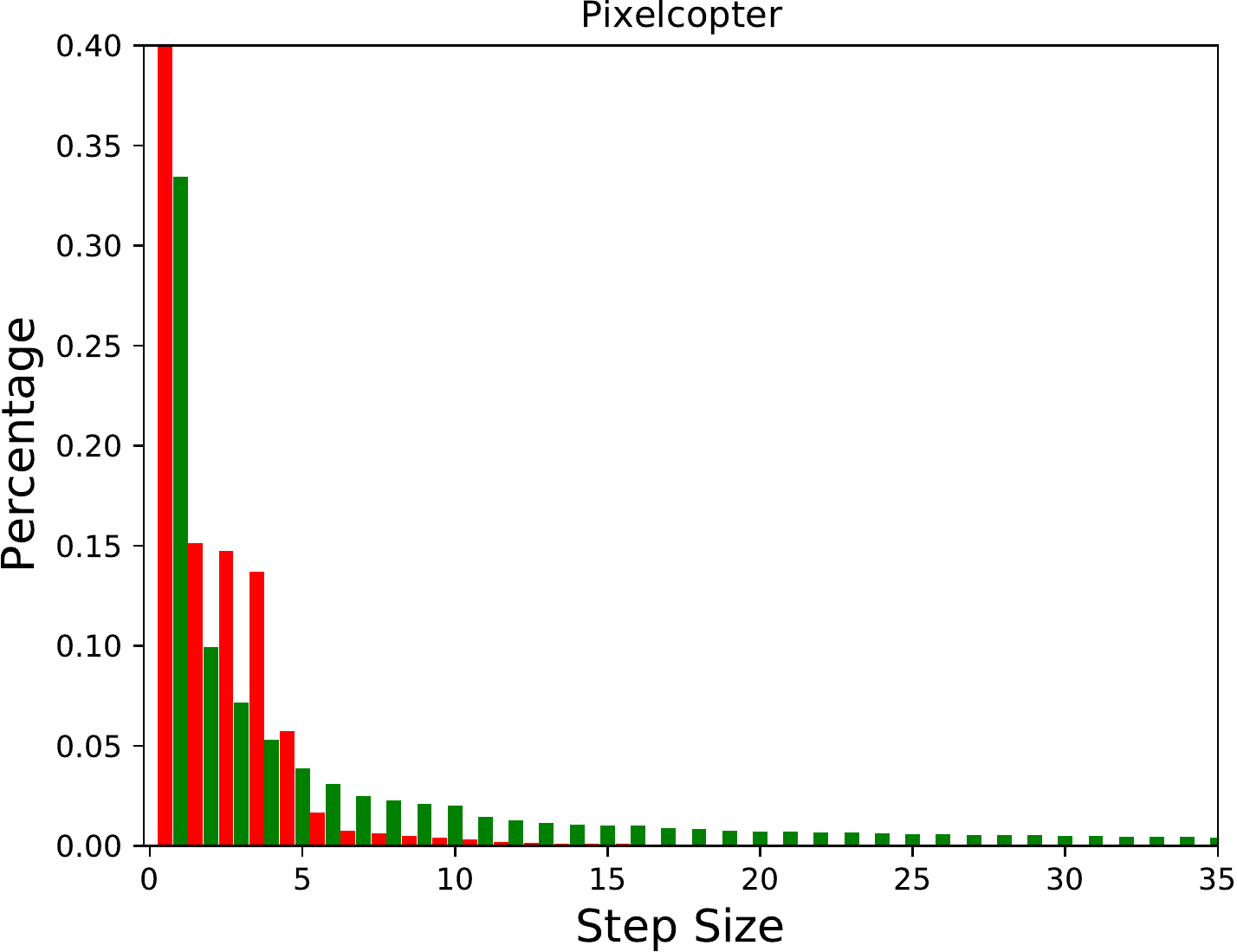}
	    	   	\includegraphics[width=\widthproperty\linewidth]{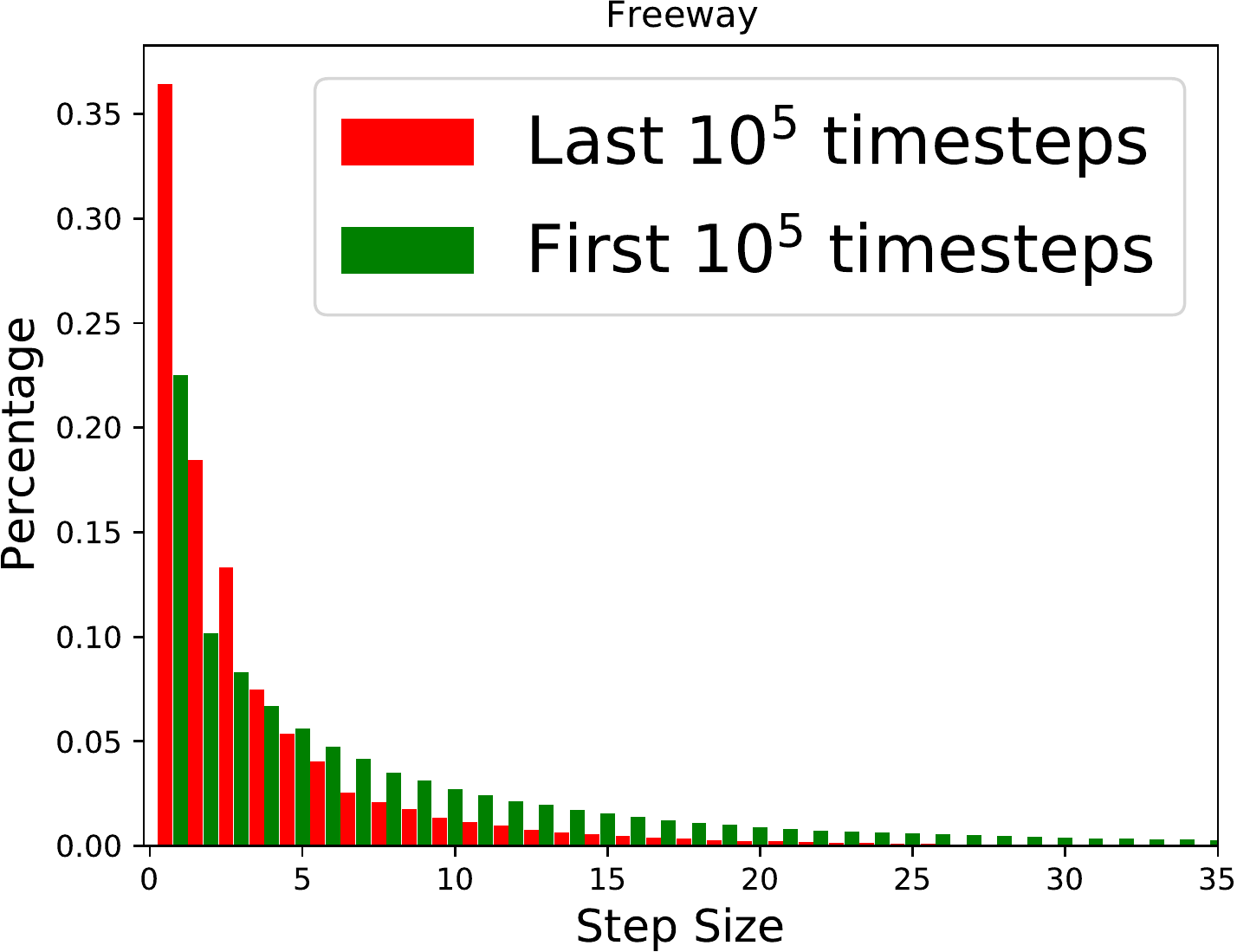}
      	}
	\ifshowfig
    \caption{
    	(Top) \textbf{The mean of chosen \steps/}, i.e., $\argmax_{ N' } \nstepReturn[N'][\tau]$, of \greedyStepDQN/ during training process. 
    	(Bottom) \textbf{The statistics of the chosen \steps/} of \greedyStepDQN/. We make statistics of the chosen \step/ size over the first $10^5$ timesteps (green) and and the last $10^5$ timesteps (red) during training process. 
    }
    \else
        \caption{
        }
    \fi\label{fig_stepsize}
\end{figure}
\fi

\textbf{Adaptive \step/.}
Our \greedyStepDQN/ can adaptively decide the \step/ based on the ``optimality'' principle, that is, $\argmax_{ n } \nstepReturn[n][\tau]$.
\Cref{fig_stepsize} (Top) shows the mean of chosen \steps/ of \greedyStepDQN/ during the training process, i,e, $\argmax_{ N' } \nstepReturn[N'][\tau]$.
\Cref{fig_stepsize} (Bottom) plots the statistics of chosen \steps/ at different phase of training process. 
As can be seen, at the beginning training phase, the algorithm tends to choose relatively larger \steps/. \T{While at the latter training phase, shorter \steps/ is preferred.}
This mechanism can make the algorithm fast propagate the information in the data at the beginning.
While at the latter training phase, as the value function has fitted well and no better trajectory data is provided, the chosen \steps/ are getting smaller automatically.





\section{Conclusion}
In this work, we introduce two novel multi-step Bellman Optimality Equations for efficient information propagation.
We prove that the solution of the equations is the optimal value function, and the corresponding operators generally converge faster than the traditional Bellman Optimality operator.
The derived algorithms have several advantages than the existing off-policy algorithms.
The feasibility and effectiveness of the proposed method have been demonstrated on a series of standard benchmark datasets with promising results.

\bibliography{GreedyMultistep}
\bibliographystyle{plain}

\ifISWORD
\else
\clearpage
\includepdfmerge{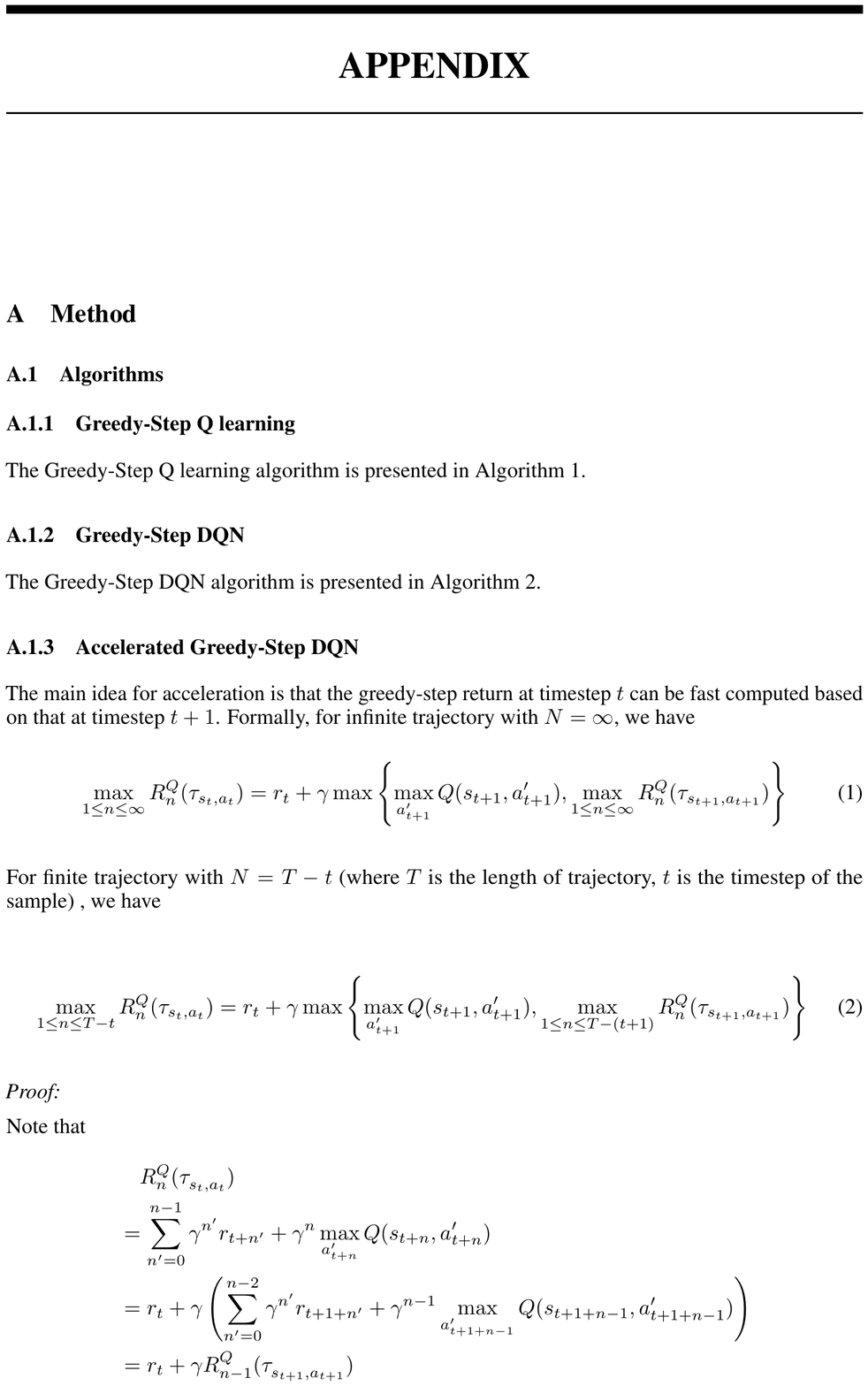,-}    
\fi

\end{document}